\documentclass{article}

\PassOptionsToPackage{sort, numbers}{natbib}

\usepackage[preprint]{neurips_2022}

\usepackage[utf8]{inputenc} 
\usepackage[T1]{fontenc}    
\usepackage{hyperref}       
\usepackage{amsmath}
\usepackage{url}            
\usepackage{booktabs}       
\usepackage{amsfonts}       
\usepackage{nicefrac}       
\usepackage{microtype}      
\usepackage{xcolor}         
\usepackage[english]{babel}
\usepackage{amsthm}
\usepackage{amssymb}
\usepackage{graphicx}
\usepackage{subfig}
\usepackage{mathtools}
\usepackage{tablefootnote}
\usepackage{float}
\usepackage{soul}
\usepackage{xcolor}
\usepackage{verbatim}
\usepackage{enumitem}
\usepackage[export]{adjustbox}
\def\eqref#1{(\ref{#1})}

\title{ComGAN: Toward GANs Exploiting Multiple Samples}

\author{%
  Haeone~Lee \\
  Independent Researcher \\
  \texttt{poiroth946@gmail.com} \\
}

\begin{document}
\newtheorem{theorem}{Theorem}
\newtheorem{proposition}{Proposition}
\newtheorem{lemma}{Lemma}

\maketitle

\begin{abstract}
In this paper, we propose \emph{ComGAN}(ComparativeGAN) which allows the generator in GANs to refer to the semantics of comparative samples(e.g. real data) by comparison. ComGAN generalizes relativistic GANs by using arbitrary architecture and mostly outperforms relativistic GANs in simple input-concatenation architecture. To train the discriminator in ComGAN, we also propose \emph{equality regularization}, which fits the discriminator to a neutral label for equally real or fake samples. Equality regularization highly boosts the performance of ComGAN including WGAN while being exceptionally simple compared to existing regularizations. Finally, we generalize comparative samples fixed to real data in relativistic GANs toward fake data and show that such objectives are sound in both theory and practice. Our experiments demonstrate superior performances of ComGAN and equality regularization, achieving the best FIDs in 7 out of 8 cases of different losses and data against ordinary GANs and relativistic GANs.
\end{abstract}

\section{Introduction}

Generative Adversarial Networks(GANs)\citep{gan2014} train the generator $G$ to generate fake data that follows the distributions of real data by fooling the discriminator $D$ which distinguishes fake data from real data. GANs have shown promising performance in generative modeling, being able to generate photo-realistic images\citep{biggan, karras2019style, styleganv2}. However, they suffer from training instability stemming from the non-stationary adversarial game\citep{biggan, improvedtech, principled}, which might lead to mode collapse where the generated samples have very few modes or complete failure, e.g., generating only random noises. 

To improve stability, several works explored applying different loss functions from vanilla GAN\citep{ganscreatedequal} such as least square loss\citep{mao2017lsgan}(LSGAN), hinge loss \citep{lim2017geometric}(HingeGAN), and wasserstein loss\citep{arjovsky2017wasserstein}(WGAN). Especially, WGAN\citep{arjovsky2017wasserstein} minimizes the wasserstein distance induced from integral probability metrics(IPMs)\citep{ipm} which is mathematically weaker than other divergences, not causing diverging or vanishing gradients in theory. Along with WGAN, various GANs adopting losses from IPMs(IPM-GANs)\citep{fishergan, mcgan, cramergan} were also claimed to improve stability upon vanilla GAN. Apart from changing the loss function, other works regularize the discriminator to improve GANs training. The common approaches are constraining lipschitz norm\citep{wgan-gp, miyato2018spectral}, regularizing the input gradient norm\citep{dragan, r1reg}, or posing invariance over data augmentation\citep{consistencyv2, zhang2019consistency} for better conditioning of the discriminator. Still, applying regularization might be computationally demanding in the case of computing input gradient\citep{wgan-gp, dragan, r1reg}. Also, \citet{ganscreatedequal} showed that GANs using different losses do not consistently outperform vanilla GAN. 

Instead, some previous works explored utilizing multiple samples in the discriminator, which can be paired with any GAN losses. \citet{jolicoeur2018relativistic} proposed relativistic GANs where the discriminator logit is calculated by both real and fake samples after subtracting one from the other, i.e. $C(x)-C(y)$ where $x,y \sim p_{d,g}$. The author claims that relativistic GANs improve training stability by (i) reflecting prior knowledge that half of samples in mini-batch are fake and (ii) following training procedure in real divergence minimization and (iii) resembling IPM-GANs which are known to be more stable than other variants. Similarly, \citet{lin2018pacgan} proposed PacGAN where the discriminator is trained by multiple real and fake samples, matching divergence between joint distributions of n real and fake samples. PacGAN inhibits mode-collapse due to the property of n-packed joint distributions\citep{lin2018pacgan}. Similar batch-wise discrimination was also attempted in \citep{improvedtech} to alleviate mode-collapse.

In this paper, we propose ComparativeGAN(ComGAN) which generalizes the relativistic GANs toward arbitrary architecture. Departing from logit-level comparison in relativistic GANs, ComGAN allows the generator to refer to the semantics of the comparative sample(e.g. real data) when updating fake samples, mostly outperforming relativistic GANs under simple input concatenation. While \citet{dualcontrastive} also proposed to employ real data as a reference in the discriminator, their training relied on the attention mechanism, different from ComGAN which naturally includes reference samples in loss functions. In addition, we show that relativistic GAN overlooks training terms for equally real or fake data and propose \emph{equality regularization} to include such terms. Equality regularization highly boosts the performance of ComGAN, including WGAN and vanilla relativistic GAN. In contrast to gradient penalty regularizations\citep{wgan-gp, r1reg, dragan} and spectral normalization\citep{miyato2018spectral}, equality regularization is much more simple and easy to implement. Finally, we show that PacGAN and ComGAN are tied for matching either \emph{packed} or \emph{swapped} joint distributions and derive a regularization for PacGAN which corresponds to equality regularization in ComGAN. In our experiments, the combination of ComGAN and equality regularization outperformed relativistic GANs and ordinary GANs in 7 out of 8 cases under various losses and data. To summarize, our contributions are as follows:
\begin{itemize}[leftmargin=*]
\item We propose ComGAN which enables semantic comparison from relativistic GANs and demonstrate improved performance.
\item We generalize comparative samples which were fixed to real data in relativistic GANs towards fake data and show such objective is sound in both theory and practice.
\item We propose equality regularization which improves the performance of ComGAN including WGAN while being exceptionally simple compared to existing regularizations.
\item We derive a regularization objective for PacGAN corresponding to equality regularization in ComGAN, which can train a WGAN-like objective without other regularization.  
\end{itemize}
The last of this paper is organized as follows: we review related works in Section~2 and define notations of GANs in Section~3. In Section~4, we introduce ComGAN and equality regularization and analyze them in statistical divergences and optimal discriminators. We summarize our experimental results in Section~5 and conclude the paper in Section~6.

\section{Related Works}
\paragraph{Generative adversarial networks}
GANs\citep{gan2014} are powerful generative models that have been adopted in various tasks such as synthesizing photo-realistic images\citep{biggan, karras2019style, styleganv2}, super-resolution\citep{wang2018esrgan, deepunfolding, srntt}, domain translation\citep{unit, starganv2}. However, GANs training is known to be unstable, resulting in mode collapse\citep{improvedtech, principled, biggan}. To resolve this, various mechanisms have been proposed such as using different loss functions\citep{mao2017lsgan, arjovsky2017wasserstein, lim2017geometric} and applying regularizations\citep{r1reg, miyato2018spectral, zhang2019consistency, biggan, styleganv2}.

\paragraph{Regularizing discriminators}
To improve and stabilize GANs training, numerous works have focused on regularizing the discriminator. \citep{wgan-gp} proposed gradient penalty which fits the gradient norm of the discriminator to 1.0 to satisfy lipschitz constraints. Similarly, \citep{r1reg} penalizes the gradient-norm of real data in the discriminator and \citep{dragan} fits the gradient norm around real data to 1.0, which was shown to be effective to mitigate mode-collapse. \citep{miyato2018spectral} proposed spectral normalization which divides the discriminator weights by their spectral norms for capacity regularization. Instead of relying on lipschitz constraint, consistency regularization\citep{zhang2019consistency, consistencyv2} imposes the discriminator output to be consistent under different data augmentation, reducing computational overhead. Our equality regularization in ComGAN is much more simple and easy to implement compared to gradient penalties and spectral normalizations. Also, it is free from choosing proper data augmentation suitable for training data compared to consistency regularization. 

\paragraph{Employing multiple samples in discriminators}
Instead of employing regularizations, other works improve GANs by modifying discriminator structures to exploit multiple samples. \citep{jolicoeur2018relativistic} proposed relativistic GANs which calculate the discriminator logit by using both real and fake samples after subtracting one from the other. Relativistic GANs were claimed to improve training stability\citep{jolicoeur2018relativistic} and have been adopted in various tasks such as super-resolution\citep{wang2018esrgan, deepunfolding} and deblurring\citep{kupyn2019deblurgan}. Similarly, \citep{lin2018pacgan} proposed PacGAN which trains a discriminator by multiple real samples and fake samples to surpass mode-collapse using the property of joint distributions. \citep{improvedtech} proposed mini-batch discrimination which is similar to PacGAN except for using more specialized architecture. Our work generalizes over relativistic GANs and permits the use of arbitrary architecture, enabling semantic comparison to provide better learning signals. We also show the connections between relativistic GANs and PacGAN in statistical divergences and optimal discriminators.

\paragraph{Employing reference samples}
Since our ComGAN enables semantic comparison between real data and fake data using arbitrary architecture, the generator is able to refer to real data when updating fake samples. Similarly, \citep{dualcontrastive} proposed GANs that employ real data as a reference in a discriminator using an attention mechanism in order to better guide the generator. However, their reference samples are unrelated to the loss function and complicated attention mechanism thus had to be used to combine reference samples in the training. Reference-based super-resolution(RefSR)\citep{srntt, deformconvsr, crossnet} also uses reference samples to synthesize an HR image from an LR image, given the information from a reference image(e.g. similar scenes). In RefSR, such information is directly synthesized to the generator in various ways, e.g., style transfer\citep{srntt}, warping\citep{crossnet}, deformable convolution\citep{deformconvsr} while ComGAN indirectly leverages reference samples by comparison in the discriminator. Finally, in text generation, \citep{selfadversarial} proposed \emph{self-improvement} which trains the generator using reinforcement learning by rewarding the fake sentences that are better than previously generated ones and the other way around for worse sentences. Self-improvement was claimed to solve reward sparsity issues and prevent mode collapse\citep{selfadversarial}. In fact, self-improvement is equivalent to employing fake data as a reference in ComGAN and we expect that ComGAN using fake data would have similar advantages.

\section{Background}
\label{headings}
\subsection{Generative adversarial networks}
GANs training involves an adversarial game between two models, the discriminator $D_\psi$ and the generator $G_\theta$, which are neural networks parameterized by $\psi$ and $\theta$. We express the training objectives of $D$ and $G$ as follows(we omitted the expression of $\psi$ and $\theta$ for simplicity):
\begin{align}
\label{eq:vanilla_d}
L_D (\text{min})&=E_{x\sim{p_d}}[f_1(C(x))] + E_{x\sim{p_g}}[f_2\left(C(x)\right)] \\ 
\label{eq:vanilla_g}
L_G (\text{min})&=E_{x\sim{p_d}}[g_1(C(x))] + E_{x\sim{p_g}}[g_2(C(x))]
\end{align}
where $C(x)$ denotes a discriminator logit function which corresponds to a pre-activation value of the discriminator\footnote{We also use the term ``discriminator" to refer to $C(x)$}, i.e. $D(x)=\mathcal{A}(C(x))$ for output activation $\mathcal{A(\cdot)}$, and $f_1$, $f_2$, $g_1$, $g_2$ are scalar to scalar functions depending on the loss function applied, and $p_d$, $p_g$, $p_z$ are real data distribution, fake data distribution, and generator prior(e.g. gaussian) respectively. We assume that both $L_D$ and $L_G$ are minimized unless otherwise specified by $L_D(\text{max})$ and $L_G(\text{max})$ which are $-L_D(\text{min})$ and $-L_G(\text{min})$. By setting $f_1(x)=\log(1+e^{-x})$ and $f_2(x)=\log(1+e^x)$, one can recover vanilla discriminator loss used in SGAN(standard-GAN)\citep{gan2014}, expressed by
\begin{equation*}
L_D=E_{x\sim{p_d}}[-\log D(x)]+E_{x\sim{p_g}}[-\log(1 - D(x))] \tag{SGAN}
\end{equation*}
where $D(x)=\sigma(C(x))$ and $\sigma(\cdot)$ denotes sigmoid function. In such a case, $D(x)$ can be seen as a probability that $x$ is real. For the generator objective, setting $g_1=-f_1$, $g_2=-f_1$ and $g_1=f_2$, $g_1=f_2$ respectively correspond to saturating loss and non-saturating loss\citep{gan2014, jolicoeur2018relativistic}. Under the assumption of the optimal discriminator, GANs can also be interpreted as a divergence minimization between real data and fake data, e.g., Jensen-Shannon divergence(JSD) in SGAN\citep{gan2014}, {\it f}-divergence\citep{fgan} or IPMs(Integral Probability Metrics)\citep{arjovsky2017wasserstein, otgan, mcgan}. Normally, IPM-GANs are known to have benefits in stability due to the desirable properties of  their metrics(e.g. weak divergence)\citep{arjovsky2017wasserstein, mcgan}.    

\subsection{Relativistic GANs}
Relativistic GANs modify the discriminator structure in GANs to improve stability\citep{jolicoeur2018relativistic}. Relativistic GANs are divided into RGAN and RaGAN by using two samples or multiple samples, whose discriminator objectives are given by
\begin{align*}
L_D &= E_{x\sim{p_d}, z\sim{p_z}}[f_1(C(x)-C(G(z)))+f_2(C(G(z))-C(x))] \tag{RGAN} \\
L_D &= E_{x\sim{p_d}}[f_1(C(x)-\overline{C(x_f)})] + E_{z\sim{p_z}}[f_2(C(G(z))-\overline{C(x_r)})] \tag{RaGAN}
\end{align*}
where $\overline{C(x_f)}=E_{z\sim{p_z}}[C(G(z))]$ and $\overline{C(x_r)}=E_{x\sim{p_d}}[C(x)]$. The generator losses are defined using the same $g_1$ and $g_2$ in ordinary GANs. Therefore, RGAN and RaGAN are equivalent to substituting the logit value in ordinary GANs with the value after subtracting the logit of opponent data. Note that WGAN\citep{arjovsky2017wasserstein} can be seen as a special case of RGAN and RaGAN, which corresponds to $f_1=-I$ and $f_2=I$ without consideration of capacity constraints(e.g. weight clipping\citep{arjovsky2017wasserstein}, gradient penalty\citep{wgan-gp}, etc.).

\section{Method}
\label{others}
\subsection{ComGAN}
To introduce our method of ComGAN, we first note that the discriminator objective of RGAN can be re-expressed by 
\begin{equation*}
L_D=E_{x,y\sim{p_{d,g}}}[f_1(C(x)-C(y))]+E_{x,y\sim{p_{g,d}}}[f_2(C(x)-C(y))]
\end{equation*}
where $p_{d,g}$ is the joint distribution of two independent real and fake data, and $p_{g,d}$ is similarly defined. A similar objective also holds for the generator. Instead of sticking to the subtraction form of $C(x)-C(y)$, we generalize a discriminator structure by arbitrary $C(x, y)$, resulting in the following:
\begin{align}
\label{eq:comgan_d}
L_D&=E_{x,y\sim{p_{d,g}}}[f_1(C(x,y))]+E_{x,y\sim{p_{g,d}}}[f_2(C(x,y))] \\ 
\label{eq:comgan_g}
L_G&=E_{x,y\sim{p_{d,g}}}[g_1(C(x,y))]+E_{x,y\sim{p_{g,d}}}[g_2(C(x,y))]
\end{align}
where $C(x,y)$ is the discriminator logit function. Note that RGAN is a special case of (3) and (4), corresponding to letting $C(x,y)=\phi(x)- \phi(y)$ where $\phi:\mathbb{R}^{n_d} \rightarrow \mathbb{R}$ and $n_d$ is the dimension of input data. In the case of using cross-entropy loss to train the discriminator, (3) is equal to
\begin{equation}
\label{eq:scomgan_d}
L_D = E_{x,y\sim{p_{d,g}}}[-\log D(x,y)]+E_{x,y\sim{p_{g,d}}}[-\log(1 - D(x,y))]
\end{equation}
where $D(x,y)=\sigma(C(x,y))$ and $\sigma(\cdot)$ denotes sigmoid function. Intuitively, training the discriminator to minimize \eqref{eq:scomgan_d} tasks $D(x,y)$ to compare the `reality' of $x$ and $y$ in the range of 0-1 and $D(x, y)$ can be understood as a probability that $x$ is more real than $y$ (i.e. $D(x, y)\approx 1$ indicates $x$ is more real than $y$ and vice-versa). Motivated by this, we name GANs using equations \eqref{eq:comgan_d} and \eqref{eq:comgan_g} \emph{ComGAN}(ComparativeGAN). ComGAN allows the use of arbitrary discriminator structure for $C(x, y)$ and includes RGAN as its special case. Such property is useful since advanced neural network architecture(e.g. attention) can be utilized to compare $x$ and $y$(see Section~4.2). Meanwhile, in SComGAN(standard-ComGAN) which employs the objective of \eqref{eq:scomgan_d}, the optimal discriminator $D^*$ is given by
\begin{equation*}
    D^*(x,y)=\frac{p_{d,g}(x,y)}{p_{d,g}(x,y)+p_{g,d}(x,y)} 
\end{equation*}
Intuitively, we see that $D^*(x,y)$ imposes the value close to 1 if sample $(x,y)$ has a high likelihood of $p_{d,g}$ and a low likelihood of $p_{g,d}$ and the other way around for $p_{g,d}$. The following proposition gives a further intuition of the behavior of ComGAN.
\begin{proposition}
The optimal discriminator of SComGAN is equal to $\sigma(C^*(x)-C^*(y))$ where $C^*$ is the optimal discriminator logit function of SGAN.
\end{proposition}
\begin{proof}
Note that the optimal discriminator of SGAN is expressed by
\[D^*(x)=\frac{p_d(x)}{p_d(x)+p_g(x)}=\sigma\left(\log\frac{p_d(x)}{p_g(x)}\right). \hspace{0.1cm} \therefore \hspace{0.1cm} C^*(x)=\log\frac{p_d(x)}{p_g(x)}\]
Meanwhile, 
\[D^*(x,y)=\sigma\left(\log\frac{p_{d,g}(x,y)}{p_{g,d}(x,y)}\right)=\sigma\left(\log\frac{p_d(x)}{p_g(x)}-\log\frac{p_d(y)}{p_g(y)}\right) \qedhere \]
\end{proof}
Thus, the optimal discriminator of SComGAN equals subtracting the logits of the optimal SGAN discriminator with respect to two input samples, before applying the sigmoid function. In the case of RGAN where $C(x,y)=\phi(x)- \phi(y)$, we conclude $\phi^*(x)=C^*(x)+k \text{ (constant)}$. Meanwhile, in the case of LSGAN using label 1/-1, we observe a conflict between the optimal discriminator given by $D^*(x,y)=\frac{p_{d,g}(x,y)-p_{g,d}(x,y)}{p_{d,g}(x,y)+p_{g,d}(x,y)}$\citep{mao2017lsgan} and an RGAN discriminator $D(x,y)=\phi(x)-\phi(y)$($\because$ $\mathcal{A}=I$) since the former cannot be expressed by the latter. 

In divergence minimization, ComGAN matches the divergence between $p_{d,g}$ and $p_{g,d}$ as shown in Lemma 1.  
\begin{lemma}
Suppose $L_G$ in \eqref{eq:vanilla_g} is equal to $D(p_d \parallel p_g)$ under the optimal discriminator $C^*(x)$. $L_G$ in \eqref{eq:comgan_g} defined by the same $f_1$, $f_2$, $g_1$, $g_2$ is equal to $D(p_{d,g} \parallel p_{g,d})$ under $C^*(x,y)$.
\end{lemma}
\begin{proof} \eqref{eq:comgan_d} and \eqref{eq:comgan_g} are equal to substituting $p_d$ and $p_g$ with $p_{d,g}$ and $p_{g,d}$ in \eqref{eq:vanilla_d} and \eqref{eq:vanilla_g} and $p_{d,g}$ and $p_{g,d}$ are probability distributions. \qedhere
\end{proof}

Therefore, SComGAN and WComGAN minimize $JSD(p_{d,g} \parallel p_{g,d})$ and $W^1(p_{d,g} \parallel p_{g,d})$ respectively\citep{gan2014, arjovsky2017wasserstein}.

\subsection{Comparative sample}
In ComGAN, fake samples are trained to be considered more real than real samples in the comparison by the discriminator, aiming to increase the value of $C(G(z),y)$ where $y \sim{p_d}$. We refer to $y$ as \emph{comparative sample}, which is a reference employed to assist the training of $G(z)$. We suppose that the `good' discriminator should compare two samples in high-level semantics to avoid small details of comparative sample from affecting the update of $G(z)$. In such aspects, the subtraction structure of the discriminator in RGAN conducts comparison in a discriminator logit level, the highest semantics that can be captured by the discriminator(i.e. whether to be real or fake). Although this is desirable due to the aforementioned property, RGAN renders it impossible for the generator to directly utilize information from comparative samples such as textures or frequent patterns. In fact, our experiments show that simply using $x$ and $G(z)$ as network input by concatenation outperforms RaGAN in SComGAN and HingeComGAN. Taking $G(z)$ and $y$ as the network input might also help to alleviate the mode-collapse issue since the limited modes in $G(z)$ would easily be captured from the comparison to real data. Similarly, \citet{dualcontrastive} proposed reference attention where the attention map is computed upon real samples and is applied to the main layers of the discriminator as a reference. Such training objective can be expressed by $E_{x,y\sim{p_{d,d}}}[-\log D(x,y)]+E_{x,y\sim{p_{g,d}}}[-\log(1-D(x,y))]$ in the case of SGAN with attention mechanism employed in $D(x,y)$. In this case, the network might learn to ignore a reference $y$ since the input only differs by $x$ and the optimal discriminator does not change from vanilla SGAN. On the other hand, applying such attention-based architectures in ComGAN is promising since ComGAN naturally leverages comparative samples in the objective. In addition, ComGAN can be extended toward multiple comparative samples(i.e. comparing a real sample to the batch of fake samples and vice-versa) in order to broadly capture data semantics by using the discriminator $C(x,x_1,\dots,x_n)$ where $x$ is the main input and $x_1 \dots x_n$ are comparative samples. In case that $C(x,x_1,\dots,x_n)=\phi(x)-\frac{1}{n}\sum^{n}_{i=1}{\phi(x_i)}$, we recover RaGAN-like objective where mean logits are estimated using n samples yet other architectures such as simply concatenating $x,x_1 \dots, x_n$ for the network input would be possible.

Meanwhile, it should be noted that comparative can be generalized beyond real data as long as it can guide $G(z)$ to be improved upon itself. In previous work, \citet{selfadversarial} proposed \emph{self-improvement}, a reinforcement learning-based text generation algorithm that rewards the generator for generating sentences that are better than its previously generated samples, distinguished by the discriminator. Self-improvement was claimed to address the reward scarcity issue since improving upon fake samples is easier than improving upon real samples\citep{selfadversarial}. Employing fake data as comparative samples in ComGAN might enjoy the same advantages, e.g., preventing gradient-saturating issues, although the ability to refer to the semantics of real data will be lost. In the same sense, we consider using $G(z)$ itself as a comparative sample. From now on, we refer to ComGAN adopting fake data and the same sample(i.e. $G(z)$) for comparative samples as ComFakeGAN and ComSameGAN and express the generator objectives as follows(note that the discriminator objective remains unchanged):
\begin{equation*}
L_G=E_{x,z\sim{p_{g,z}}}[g_1(C(\overline{x}, G(z)))]+E_{z,x\sim{p_{z,g}}}[g_2(C(G(z), \overline{x}))]
\tag{ComFakeGAN}
\end{equation*}
\begin{equation*}
L_G=E_{z\sim{p_z}}[g_1(C(\overline{G(z)}, G(z)))]+E_{z\sim{p_z}}[g_2(C(G(z), \overline{G(z)}))]
\tag{ComSameGAN}
\end{equation*}
where $\overline{x}$ indicates that the training is not performed with respect to $x$. In the case of non-saturating loss, i.e. $g_1=f_2$ and $g_2=f_1$, the objective of ComFakeGAN is equivalent to replacing the real data part of the joint distributions in equation \eqref{eq:comgan_d} with fake data. Since the discriminator is trained to minimize \eqref{eq:comgan_d} with respect to $p_{d,g}$ and $p_{g,d}$, minimizing such an objective with respect to fake data used in place of real data is sound. Similarly, ComSameGAN can be seen as improving $G(z)$ upon itself. In statistical divergences, those objectives can be analyzed as Theorem 1. 
\begin{theorem}
Under the optimal discriminator, the following holds for derivatives of generator objectives with respect to the generator parameter $\theta$ in saturating SComFakeGAN, non-saturating SComFakeGAN, saturating SComSameGAN, and non-saturating SComSameGAN:
\begin{equation*}
\begin{alignedat}{3}
&\nabla_\theta L^{sat}_{SComFakeGAN}&&+\nabla_\theta L^{nonsat}_{SComFakeGAN}&&=2\nabla_\theta KL(p_g \parallel p_d) \\
&\nabla_\theta L^{sat}_{SComSameGAN}&&=\nabla_\theta L^{nonsat}_{SComSameGAN}&&=\nabla_\theta KL(p_g \parallel p_d)
\end{alignedat}
\end{equation*}
\end{theorem}
\proof See Appendix A.2

Therefore, our objectives are theoretically sound being related to minimizing KL divergence. In our experiments, ComFakeGAN using input concatenation surpassed RaGAN in SGAN and HingeGAN and showed similar performance to RaGAN in LSGAN, performing on par with ComGAN, while ComSameGAN failed in most cases. We speculate that the training signals given by $G(z)$ might be insufficient to train the same sample since the useful information not included in $G(z)$ will not be referred to.

\subsection{Equality regularization}
While the discriminator of SComGAN outputs a probability that $x$ is more real than $y$, equation \eqref{eq:scomgan_d} does not take the case that $x$ is equally real as $y$ into account since the training objective is confined to $x,y \sim p_{d,g}$ or $x,y \sim p_{g,d}$. To elicit more accurate training, we propose to fit $D(x, y) \approx 0.5$ in case of $x,y \sim{p_{d,d}}$ and $x,y \sim{p_{g,g}}$ as a form of regularization, expressed by
\begin{equation}
\label{eq:scomgan-eq}
\begin{split}
L_D=E_{x,y\sim{p_{d,g}}}[-\log D(x,y)]&+E_{x,y\sim{p_{g,d}}}[-\log (1-D(x,y))] \\+\lambda_{reg}E_{x,y\sim{p_{d,d}}}[CE(0.5 \parallel D(x,y))]&+\lambda_{reg}E_{x,y\sim{p_{g,g}}}[CE(0.5 \parallel D(x,y))]
\end{split}
\end{equation}
where $CE(p \parallel q)$ denotes cross-entropy between bernoulli distributions with a probability of $p$ and $q$ and $\lambda_{reg}$ controls the strength of the regularization(normally we set $\lambda_{reg}=1$). Note that the generator loss remains unchanged. For similar work, we account Cutmix\citep{yun2019cutmix} which improves classifier training by mixing patches from different images with their labels properly interpolated. Likewise, we perturb input (real, fake) and (fake, real) to (real, real) and (fake, fake) and impose a label of 0.5. For losses other than cross-entropy, we train $C(x,y) \approx 0$ since it corresponds to a neutral label or decision boundary in most GANs\footnote{This holds in all the GAN loss we use in experiments, e.g., SGAN, HingeGAN, LSGAN with -1/1 coding, WGAN.}, resulting in
\begin{equation}
\label{eq:comgan-eq}
\begin{split}
L_D=E_{x,y\sim{p_{d,g}}}[f_1(C(x,y))]&+E_{x,y\sim{p_{g,d}}}[f_2(C(x,y))]\\
+\lambda_{reg}E_{x,y\sim{p_{d,d}}}[\|C(x,y)\|^2]&+\lambda_{reg}E_{x,y\sim{p_{g,g}}}[\|C(x,y)\|^2]
\end{split}
\end{equation}
where $\| \cdot \|$ denotes l2 norm. We refer to newly introduced terms in \eqref{eq:scomgan-eq} and \eqref{eq:comgan-eq} as equality regularization and express ComGAN trained by \eqref{eq:scomgan-eq} or \eqref{eq:comgan-eq} as ComGAN-eq.   
\begin{proposition}
Under $\lambda_{reg}=1$, $D^*(x,y)$ of SComGAN-eq in \eqref{eq:scomgan-eq} equals $\frac{1}{2}+\frac{1}{2}(D^*(x)-D^*(y))$ where $D^*(x)$ is the optimal discriminator of SGAN. $D^*(x,y)$ of LSComGAN-eq in \eqref{eq:comgan-eq} equals $\frac{1}{2}(D^*(x)-D^*(y))$ where $D^*(x)$ is the optimal discriminator of LSGAN.
\end{proposition}
\proof See Appendix A.3

 Compared to Proposition 1, we see that SComGAN-eq corresponds to subtracting the output of SGAN discriminator(i.e. after the sigmoid) while vanilla SComGAN subtracts the logits(i.e. prior to the sigmoid). Interestingly, as LSComGAN-eq subtracts the output of LSGAN discriminator, it accords with LSRGAN architecture using $D(x,y)=\phi(x)-\phi(y)$. In our experiments, applying equality regularization significantly boosted performance in both RGAN and ComGAN using input-concatenation architectures. Especially, SRGAN-eq achieved the lowest Fr\'echet Inception Distance(FID)\citep{fid} among SGAN variants in CIFAR10\citep{cifar10} and LSComGAN-eq achieved the lowest FID in LSGAN variants of Tiny-ImageNet\citep{tinyimg}. Meanwhile, in the case of using multiple comparative samples(i.e. a batch of real samples $x_r$ and a batch of fake samples $x_f$), we propose to regularize $C(x,x_r) \approx 0$ where $x \sim{p_d}$ and $C(x,x_f) \approx 0$ where $x \sim{p_g}$, asking $x\sim{p_d}$ to be aligned with $x_r$ and similarly for $x\sim{p_g}$. In RaGAN, which currently is the only algorithm that employs multiple comparative samples, this results in the following:
\begin{equation}
\label{eq:ragan-eq}
\begin{split}
L_D=E_{x\sim{p_d}}[f_1(\phi(x)-E_{y\sim{p_g}}[\phi(y)])]&+E_{x\sim{p_g}}[f_2(\phi(x)-E_{y\sim{p_d}}[\phi(x)])]\\
+\lambda_{reg}E_{x\sim{p_d}}[\|\phi(x)-E_{y\sim{p_d}}[\phi(x)]\|^2]&+\lambda_{reg}E_{x\sim{p_g}}[\|\phi(x)-E_{y\sim{p_g}}[\phi(y)]\|^2].
\end{split}
\end{equation}
In our experiments, RaGAN-eq defined by \eqref{eq:ragan-eq} outperformed RaGAN by large margins demonstrating the effect of equality regularization. As WGAN can be seen as a special case of RaGAN, WGAN-GP-eq which uses $f_1=-I$ and $f_2=I$ in \eqref{eq:ragan-eq} also outperformed WGAN-GP\citep{wgan-gp}, achieving the lowest FID in Tiny-ImageNet experiments. Furthermore, we show that objectives \eqref{eq:comgan-eq} and \eqref{eq:ragan-eq} can be trained under simple $f_1=-I$ and $f_2=I$, which is similar to WGAN yet without regularizing discriminator capacity. We refer to such objectives as WComGAN-eq and WGAN-eq respectively, although the mathematical connection to Wasserstein distance is not proven. WComGAN-eq outperformed both WComGAN-GP and WComGAN-GP-eq and WGAN-eq obtained the lowest FID of 24.86 in our CIFAR10 experiments. Finally, we show that LSRGAN and LSRaGAN indirectly include equality and rf regularization in their objectives. To see this, note that the discriminator in RGAN is trained to satisfy $\phi(x)-\phi(y) \approx 1$ for $x,y \sim{p_{d,g}}$, which is equivalent to letting $\phi(x) \approx c_1$ and $\phi(x_f) \approx c_2$ for $c_1-c_2=1$ and $x,y \sim{p_{d,g}}$. In such a case, $\phi(x)-E_{x \sim p_d}[\phi(x)] \approx 0$ already holds for $x \sim p_d$. A similar result can be derived for LSRaGAN. This is also supported by our experiments where LSRaGAN shows notably better performance than other RaGANs and adding equality regularization to LSRaGAN rarely affects the performance. Still, the performance of LSRaGAN was lower than WGAN-eq.

Meanwhile, one might find \eqref{eq:ragan-eq} to be similar to LeCam regularization\citep{lecam} which was introduced to assist GANs training under limited data by matching the discriminator logit of current data to the moving average value of the opposite data. However, \eqref{eq:ragan-eq} matches the discriminator logits to the mean logit of the same data(not a moving average), and the mean logit $E_{x\sim{p_d}}[\phi(x)]$ and $E_{z\sim{p_z}}[\phi(G(z))]$ are also trainable(not a constant target). In the case of fixing $E_{x\sim{p_d}}[\phi(x)]$ and $E_{z\sim{p_z}}[\phi(G(z))]$ as constants denoted by $\alpha_r$ and $\alpha_f$ and assuming $\alpha_r=-\alpha_f$ in the discriminator, the generator objective of WGAN-eq becomes $(\frac{1}{2\lambda}+\alpha_r)\Delta(p_d \parallel p_g)$ in optimality where $\Delta(P \parallel Q)$ denotes LeCam divergence\citep{lecam, le2012asymptotic}(see Appendix B). 

\subsection{Regularizing PacGAN}
Similar to ComGAN, PacGAN\citep{lin2018pacgan} uses multiple samples in the discriminator. We hereby show that PacGAN is comparable to ComGAN with respect to logit operation and derive the corresponding regularization. Firstly, the objective of PacGAN is defined by
\begin{equation*}
L_D = E_{x_1\dots x_n\sim{p^n_d}}[f_1(C(x_1\dots x_n))]+E_{z_1\dots z_n\sim{p^n_g}}[f_2(C(x_1\dots x_n)))] \tag{PacGAN}
\end{equation*}
where $p_d^n$ and $p_g^n$ denote the joint distribution of $n$ independent real and fake samples\citep{lin2018pacgan}. Therefore, PacGAN matches the divergence between $p_d^n$ and $p_g^n$ while ComGAN matches $p_{d,g}$ and $p_{g,d}$.  \citet{lin2018pacgan} showed that the higher $n$ is, the more the generator is penalized against mode-collapse due to the property of joint distributions. For SPacGAN, the optimal discriminator is expressed by the sum of ordinary discriminators as Proposition 3.
\begin{proposition}
The optimal discriminator of SPacGAN, $D^*(x_1\dots x_n)$, is equal to $\sigma\left(\sum_{i=1}^{n} C^*(x_i)\right)$ where $C^*$ is the optimal discriminator logit function of SGAN.
\end{proposition}
\proof See Appendix A.4

For $n=2$, this is comparable to Proposition 1 in ComGAN with respect to adding two logits or subtracting one logit from the other. In case we let $C(x_1 \dots x_n)=\sum^{n}_{i=1} \phi(x_i)$, deviating from the original implementation where the discriminator uses the concatenation of $x_1 \dots x_n$ as input\citep{lin2018pacgan}, we obtain $\phi^*(x)=C^*(x)$. Such architecture is invariant to input permutation, satisfying $C(x,y)=C(y,x)$ when $n=2$, compared to the property of $C(x,y)=-C(y,x)$ in RGAN.

Motivated by equality regularization in ComGAN, we now study the objectives that pose $C(x,y) \approx 0$ when $x,y \sim{p_{d,g}}$ and $x,y \sim{p_{g,d}}$ in PacGAN-2, which are expressed as follows for SPacGAN and the generalized version:
\begin{equation}
\label{eq:spacgan-rf}
\begin{split}
L_D=E_{x,y\sim{p_{d,d}}}[-\log D(x,y)]&+E_{x,y\sim{p_{g,g}}}[-\log (1-D(x,y))] \\
   +\lambda_{reg}E_{x,y\sim{p_{d,g}}}[CE(0.5 \parallel D(x,y))]&+\lambda_{reg}E_{x,y\sim{p_{g,d}}}[CE(0.5 \parallel D(x,y))]
\end{split}
\end{equation} 
\begin{equation}
\label{eq:pacgan-rf}
\begin{split}
L_D=E_{x,y\sim{p_{d,d}}}[f_1(C(x,y))]&+E_{x,y\sim{p_{g,g}}}[f_2(C(x,y))]\\
+\lambda_{reg}E_{x,y\sim{p_{d,g}}}[\|C(x,y)\|^2]&+\lambda_{reg}E_{x,y\sim{p_{g,d}}}[\|C(x,y)\|^2]
\end{split}
\end{equation} 
Since SPacGAN imposes a label of 1 for (real, real) and 0 for (fake, fake), it is reasonable to use the label of 0.5 for (real, fake) and (fake, real) which are `half-real'. From now on, we refer to added terms in \eqref{eq:spacgan-rf} and \eqref{eq:pacgan-rf} as rf regularization and PacGAN using rf regularization as PacGAN-rf.
\begin{proposition}
Under $\lambda_{reg}=1$, $D^*(x,y)$ of SPacGAN-rf in \eqref{eq:spacgan-rf} equals $\frac{1}{2}(D^*(x)+D^*(y))$ where $D^*(x)$ is the optimal discriminator of SGAN. $D^*(x,y)$ of LSPacGAN-rf in \eqref{eq:pacgan-rf} equals $\frac{1}{2}(D^*(x)+D^*(y))$ where $D^*(x)$ is the optimal discriminator of LSGAN.
\end{proposition}
\proof See Appendix A.5

Proposition 4 of PacGAN and Proposition 2 of ComGAN are closely related and the differences are whether to sum the value of SGAN discriminators or subtract one from the other. In case of letting $C(x,y)=\phi(x)+\phi(x)$, we propose to fit $C(x,x_f)\approx0$ for $x\sim{p_d}$ and $x_f$ which represents a batch of fake samples, resulting in
\begin{equation}
\label{eq:pacgan-rf2}
\begin{split}
L_D=E_{x,y\sim{p_{d,d}}}[f_1(\phi(x)+\phi(y)))]&+E_{x,y\sim{p_{g,g}}}[f_2(\phi(x)+\phi(y)))]\\
+\lambda_{reg}E_{x\sim{p_d}}[\|\phi(x)+E_{y \sim{p_g}}[\phi(y)]\|^2]&+\lambda_{reg}E_{x\sim{p_g}}[\|\phi(x)+E_{y \sim{p_d}}[\phi(y)]\|^2].
\end{split}
\end{equation}
Especially, given $f_1=-I$, $f_2=I$ in \eqref{eq:pacgan-rf2}, we obtain an objective similar to WGAN and refer to it as WGAN-rf. In our experiments, WGAN-rf showed competitive performance to SGAN and LSGAN in CIFAR10\citep{cifar10} yet its performance was generally below WGAN-eq. 

In figure~\ref{fig:regularizations}, we visualize equality regularization and rf regularization in relativistic GAN and PacGAN using a geometric fashion inspired by \citep{lim2017geometric}. Especially, we let $\phi(x)=\langle w, \Phi_\psi(x) \rangle$ where $\Phi_\psi$ is discriminator feature extractor parameterized by $\psi$ and $w$ is a normal vector of the separating hyperplane(i.e. a decision boundary)\footnote{If bias $b$ in the final layer is excluded, this covers most discriminators implemented by neural networks, although RGAN and RaGAN do not depend on $b$ which is subtracted out.}. In RaGAN-eq, by training $\langle w, \Phi(x)-E_{x\sim{p_d}}[\Phi(x)] \rangle \approx 0$ where $x \sim{p_d}$, it is equal to asking $\Phi(x)-E_{x\sim{p_d}}[\Phi(x)]$ to be orthogonal to the normal vector or parallel to separating hyperplane, and similarly for fake data. We particularly used mean-matching GAN for visualization, whose objective is given by
\begin{equation}
\label{eq:mcgan}
 L_D\text{(max)}=\max_{\|w\|_2 \leq 1, \|\psi\|_\beta \leq c}\langle w, E_{x \sim{p_d}}[\Phi_\psi(x)]-E_{x \sim{p_g}}[\Phi_\psi(x)] \rangle
 \end{equation}
where $c$ is a finite constant. Mean-matching GAN can be seen as a special case of RGAN, RaGAN, and PacGAN using $f_1=-I$ and $f_2=I$ where the optimal normal vector $w^*$ is given by $cE_{x \sim{p_d}}[\Phi(x)]-cE_{x \sim{p_g}}[\Phi(x)]$\citep{lim2017geometric, mcgan}. As shown in Figure~\ref{fig:regularizations}(a), equality regularization poses feature vectors to form a parallel shape to the decision boundary in both real and fake data, which can roughly be interpreted as pushing each feature vector toward a line that is parallel to the decision boundary(blue dotted line) passing through the mean feature vector. Likewise, we observe that rf regularization in \eqref{eq:pacgan-rf2} asks $\Phi(x)+E_{x \sim{p_g}}[\Phi(x)]$ to be orthogonal to the normal vector in real data and similarly for fake data. As shown in Figure~\ref{fig:regularizations}(b) using mean-matching GAN, it is equal to encouraging $\Phi(x)+E_{x \sim{p_g}}[\Phi(x)]$(i.e. real features translated by the mean fake features) to lie on the decision boundary.

\begin{figure}[H]
    \centering
    \subfloat[\centering equality regularization]{{\includegraphics[width=6.0cm]{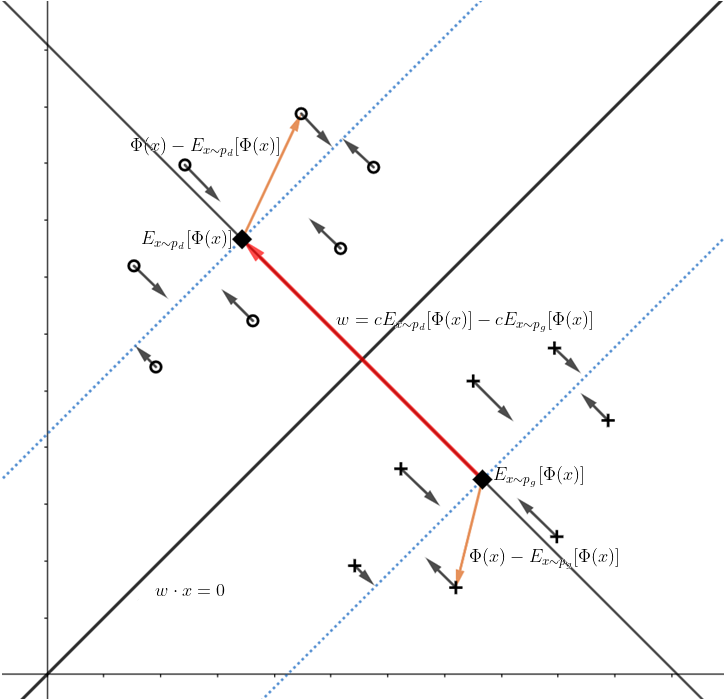} }}
    \qquad
    \subfloat[\centering rf regularization]{{\includegraphics[width=6.3cm]{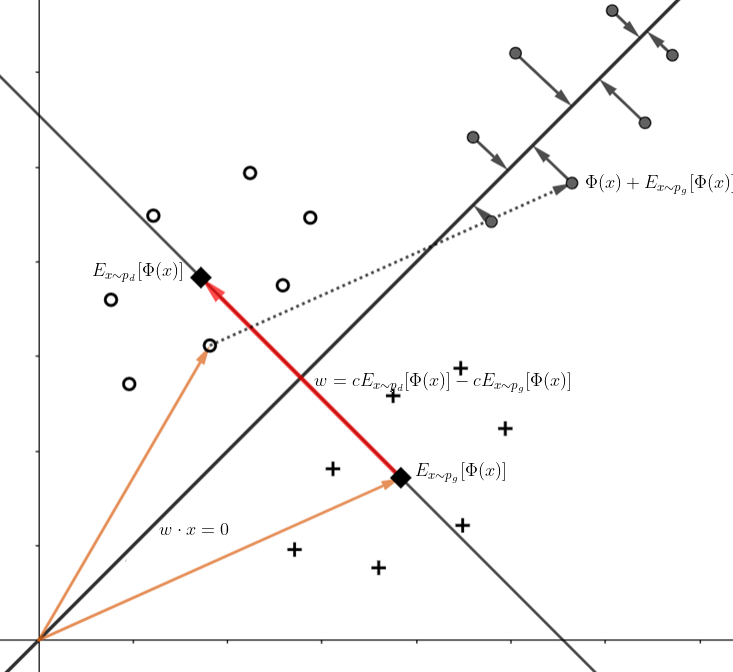} }}
    \caption{A visualization of the proposed regularizations. \textbf{(all)} Empty circles(o), cross marks(x), and filled rectangles($\blacksquare$) represent feature vectors of real data, fake data, and their mean feature vectors respectively and the black solid line represents the decision boundary with respect to the normal vector expressed by the red arrow. \textbf{(a)} equality regularization: real and fake data points are regularized to the direction of the blue dotted lines, generating a parallel shape to the decision boundary. \textbf{(b)} rf regularization: real data points transformed by the addition of the mean feature of fake data are regularized toward the decision boundary.}
    \label{fig:regularizations}
\end{figure}

\section{Experiments}
We experimented with our methods in CIFAR10 dataset\citep{cifar10} composed of 60000 $32\times32$ images(50k for training and 10k for test) and Tiny-Imagenet dataset\citep{tinyimg} composed of 100000 $64\times64$ images(90k for training and 10k for validation) using the loss functions of non-saturating SGAN\citep{gan2014}, LSGAN\citep{mao2017lsgan}, HingeGAN\citep{lim2017geometric}, and WGAN\citep{arjovsky2017wasserstein}. We implemented our algorithms upon StudioGAN framework\citep{kang2022studiogan} which provides the implementations of different GANs and various metrics for evaluation in order to assess GANs under a consistent environment with minimal changes. Our implementation is publicly available at \href{https://github.com/rl-max/PyTorch-StudioGAN}{here}\footnote{https://github.com/rl-max/PyTorch-StudioGAN}. Following the basic configuration of StudioGAN, we used the deep convolutional architecture replicated from \citep{dcgan} for training in CIFAR10 and used the ResNet architecture replicated from \citep{wgan-gp} for training in Tiny-ImageNet except for WGAN which only uses ResNet architecture. For input-concatenation architecture in ComGAN, we only modified the number of input channels and other parts of the networks were kept the same. To see further details on network architecture, we ask the readers to refer to \citep{kang2022studiogan} and \href{https://github.com/rl-max/PyTorch-StudioGAN}{our implementation}. For hyperparameters, we copied the default values of CIFAR10 training in StudioGAN and used the same values for Tiny-ImageNet experiments. Especially, we used the batch size of 64 and learning rate of 0.0002 for both discriminators and generators and set the discriminator to be updated twice per one generator update(i.e. $n_d=2$) in SGAN and LSGAN while we set $n_d=5$ in HingeGAN and WGAN. We employed Adam optimizer\citep{kingma2014adam} with $\beta_1=0.5$ and $\beta_2=0.999$ except for WGAN in Tiny-ImageNet where we use $\beta_1=0.0$ for stability. We trained SGAN, LSGAN, HingeGAN, and WGAN for 100k, 100k, 50k, and 50k steps in CIFAR10 and for 60k, 60k, 20k, and 20k steps in Tiny-ImageNet. We ran each algorithm 4 times and aggregated average performance and deviation. We also used mixed precision training\citep{micikevicius2017mixedprecision} in CIFAR10 experiments apart from WGAN. For evaluations, we calculated Fr\'echet Inception Distance(FID)\citep{fid} which is the feature-wise distance between real data and fake data, Inception-score\citep{improvedtech}, Precision/Recall\citep{precisionrecall}, and Density/Coverage\citep{dencvr} using test data in CIFAR10 and validation data in Tiny-ImageNet. Employing such various metrics for evaluation remarkably contributes to fair comparison since the performance of the model could vary along metrics\citep{kang2022studiogan}. 

\subsection{ComGAN experiments}
We first display the performance of RaGAN with a comparison to ordinary GAN in Table~\ref{table:ragan}. Note that WGAN can be seen as the special case of RaGAN.

\begin{table}[H]
  \caption{RaGAN results. We report the best FID and Inception-score during training averaged by 4 runs and Precision/Recall and Density/Coverage at the best FID with their standard deviations.}
  \label{table:ragan}
  \begin{center}
  \resizebox{\textwidth}{!}{
  \begin{tabular}{lllllll}
    \toprule
    \textbf{Algorithms} & FID $\downarrow$ & IS $\uparrow$ & Precision $\uparrow$ & Recall $\uparrow$ & Density $\uparrow$ & Coverage $\uparrow$ \\
    \toprule (CIFAR10) \\
    SGAN & 50.65($\pm$4.65) &  6.73($\pm$0.34) & 0.60($\pm$0.03) & 0.29($\pm$0.03) & 0.46($\pm$0.04) & 0.37($\pm$0.04)\\
    SRaGAN & 58.99($\pm$5.76) & 6.28($\pm$0.35) & 0.57($\pm$0.07) & 0.24($\pm$0.03) & 0.42($\pm$0.12) & 0.27($\pm$0.05)\\
    LSGAN & 47.95($\pm$17.71) & 6.71($\pm$0.57) & 0.62($\pm$0.01) & 0.38($\pm$0.15) & 0.48($\pm$0.04) & 0.39($\pm$0.11)\\
    LSRaGAN & 34.15($\pm$1.36) & \textbf{7.55}($\pm$0.15) & 0.63($\pm$0.02) & 0.51($\pm$0.01) & 0.54($\pm$0.04) & \textbf{0.50}($\pm$0.02)\\
    HingeGAN & 37.33($\pm$1.83) & 7.18($\pm$0.26) & \textbf{0.67}($\pm$0.02) & 0.33($\pm$0.01) & \textbf{0.64}($\pm$0.05) & 0.49($\pm$0.02)\\
    HingeRaGAN & 46.31($\pm$2.23) & 7.23($\pm$0.19) & 0.61($\pm$0.04) & 0.33($\pm$0.01) & 0.51($\pm$0.09) & 0.38($\pm$0.04)\\
    \textbf{WGAN-GP} & \textbf{33.20}($\pm$5.91) & 7.07($\pm$0.40) & 0.63($\pm$0.01) & \textbf{0.57}($\pm$0.05) & 0.54($\pm$0.03) & 0.49($\pm$0.07)\\
    \midrule (Tiny-ImageNet) \\
    SGAN & 76.94($\pm$3.45) & 6.99($\pm$0.24) & 0.49($\pm$0.02) & 0.11($\pm$0.03) & 0.33($\pm$0.04) & 0.22($\pm$0.01)\\
    SRaGAN & 87.13($\pm$3.17) & 6.70($\pm$0.25) & 0.42($\pm$0.05) & 0.07($\pm$0.01) & 0.22($\pm$0.05) & 0.17($\pm$0.01)\\
    LSGAN & 77.91($\pm$2.67) & 6.73($\pm$0.19) & 0.46($\pm$0.02) & 0.13($\pm$0.02) & 0.28($\pm$0.03) & 0.21($\pm$0.02)\\
    LSRaGAN & 73.70($\pm$4.29) & 6.99($\pm$0.29) & \textbf{0.52}($\pm$0.03) & 0.18($\pm$0.01) & \textbf{0.36}($\pm$0.04) & 0.23($\pm$0.02)\\
    HingeGAN & 78.46($\pm$2.02) & 7.05($\pm$0.30) & 0.48($\pm$0.06) & 0.12($\pm$0.03) & 0.31($\pm$0.06) & 0.21($\pm$0.01)\\
    HingeRaGAN & 96.38($\pm$5.87) & 6.05($\pm$0.26) & \textbf{0.52}($\pm$0.08) & 0.04($\pm$0.02) & 0.35($\pm$0.12) & 0.17($\pm$0.03)\\
    \textbf{WGAN-GP} & \textbf{62.96}($\pm$1.78) &\textbf{7.56}($\pm$0.21) & 0.47($\pm$0.01) & \textbf{0.29}($\pm$0.02) & 0.29($\pm$0.01) & \textbf{0.25}($\pm$0.01)\\
    \bottomrule 
  \end{tabular}}
  \end{center}
\end{table}

As shown in Table~\ref{table:ragan}, applying RaGAN in SGAN and LSGAN actually deteriorates the performance except for LSRaGAN which indirectly includes equality regularization. Also, WGAN-GP performs reasonably well achieving the best FIDs and Recalls in both CIFAR10 and Tiny-ImageNet datasets. In Figure~\ref{fig:ragan-gan} in Appendix D, we further show that performances of ordinary GAN and RaGAN are reversed at the later part of training as RaGAN diverges more slowly. This quite corresponds to the stability argument in the original paper\citep{jolicoeur2018relativistic}. In fact, experiments in \citep{jolicoeur2018relativistic} were conducted with stability focused, reporting FIDs of different algorithms at the specific training step(100k) and FIDs averaged over different time steps(i.e. performance variations during training). 

We also show the results of ComGAN using input-concatenation architecture in Table~\ref{table:comgan}.

\begin{table}[H]
  \caption{ComGAN using input-concatenation architecture results. Metrics were averaged by 4 runs. Values in brackets denote performance gain over RaGAN(see Table~\ref{table:cifar10} and \ref{table:tinyimg} for standard deviations).}
  \label{table:comgan}
  \begin{center}
  \resizebox{\textwidth}{!}{
  \begin{tabular}{lllllll}
    \toprule
    \textbf{Algorithms} & FID $\downarrow$ & IS $\uparrow$ & Precision $\uparrow$ & Recall $\uparrow$ & Density $\uparrow$ & Coverage $\uparrow$ \\
    \toprule (CIFAR10) \\
    SComGAN & 46.56(\textcolor{green}{-12.43}) & 6.82(\textcolor{green}{+0.54}) & 0.60(\textcolor{green}{+0.02}) & 0.32(\textcolor{green}{+0.08}) & 0.45(\textcolor{green}{+0.03}) & 0.38(\textcolor{green}{+0.11})\\
    LSComGAN & 76.03(\textcolor{orange}{+41.87}) & 5.14(\textcolor{orange}{-2.41}) & 0.65(\textcolor{green}{+0.02}) & 0.22(\textcolor{orange}{-0.3}) & 0.51(\textcolor{orange}{-0.03}) & 0.25(\textcolor{orange}{-0.24}) \\
    HingeComGAN & 37.28(\textcolor{green}{-9.03}) & 7.46(\textcolor{green}{+0.22}) & 0.65(\textcolor{green}{+0.05}) & 0.34(\textcolor{green}{+0.01}) & 0.62(\textcolor{green}{+0.11}) & 0.49(\textcolor{green}{+0.10})\\
    WComGAN-GP & 79.19(\textcolor{orange}{+45.99}) & 3.96(\textcolor{orange}{-3.10}) & 0.69(\textcolor{green}{+0.06}) & 0.09(\textcolor{orange}{-0.47}) & 0.60(\textcolor{green}{+0.06}) & 0.19(\textcolor{orange}{-0.30})\\
    \midrule (Tiny-ImageNet) \\
    SComGAN & 75.20(\textcolor{green}{-11.93}) & 7.57(\textcolor{green}{+0.87}) &  0.43(\textcolor{green}{+0.01}) & 0.13(\textcolor{green}{+0.07}) & 0.25(\textcolor{green}{+0.03}) & 0.20(\textcolor{green}{+0.04})\\
    LSComGAN & 82.58(\textcolor{orange}{+8.88}) & 6.93(\textcolor{orange}{-0.06}) & 0.45(\textcolor{orange}{-0.07}) & 0.11(\textcolor{orange}{-0.07}) & 0.28(\textcolor{orange}{-0.08}) & 0.19(\textcolor{orange}{-0.03})\\
    HingeComGAN & 85.45(\textcolor{green}{-10.93}) & 6.63(\textcolor{green}{+0.58}) & 0.43(\textcolor{orange}{-0.09}) & 0.09(\textcolor{green}{+0.05}) & 0.26(\textcolor{orange}{-0.10}) & 0.18(\textcolor{green}{+0.01})\\
    WComGAN-GP & 154.41(\textcolor{orange}{+91.45}) & 3.49(\textcolor{orange}{-4.07}) & 0.32(\textcolor{orange}{-0.15}) & 0.01(\textcolor{orange}{-0.29}) & 0.12(\textcolor{orange}{-0.17}) & 0.06(\textcolor{orange}{-0.19})\\
    \bottomrule 
  \end{tabular}}
  \end{center}
\end{table}

In Table~\ref{table:comgan}, we observe that using input concatenation architecture in ComGAN surpasses RaGAN by significant margins in SGAN and HingeGAN while it deteriorates the performance in LSGAN and WGAN(against WGAN-GP). Note that a such comparison might be unfavorable to ComGAN since RaGAN uses multiple comparative samples and LSRaGAN benefits from the implicit equality regularization effect. In fact, we show in Section~5.2 that LSComGAN with equality regularization outperforms LSRaGAN in Tiny-ImageNet and shows competitive performance in CIFAR10. Such results prove the potential of applying semantic comparison in the discriminator. 

Finally, we present the results of using fake data as the comparative sample in Table~\ref{table:comfakegan}.

\begin{table}[H]
  \caption{ComFakeGAN using input concatenation architecture results. We denote ComFakeGAN by FakeGAN.}
  \label{table:comfakegan}
  \begin{center}
  \resizebox{\textwidth}{!}{
  \begin{tabular}{lllllll}
    \toprule
    \textbf{Algorithms} & FID $\downarrow$ & IS $\uparrow$ & Precision $\uparrow$ & Recall $\uparrow$ & Density $\uparrow$ & Coverage $\uparrow$ \\
    \toprule (CIFAR10) \\
    SFakeGAN & 45.55($\pm$4.53) & 6.57($\pm$0.33) & 0.63($\pm$0.03) & 0.33($\pm$0.03) & 0.51($\pm$0.05) & 0.39($\pm$0.05)\\
    LSFakeGAN & 35.45($\pm$1.98) & 7.09($\pm$0.15) & 0.64($\pm$0.01) & 0.48($\pm$0.02) & 0.54($\pm$0.01) & 0.49($\pm$0.04) \\
    HingeFakeGAN & 36.24($\pm$1.72) & 7.42($\pm$0.07) & 0.63($\pm$0.02) & 0.36($\pm$0.03) & 0.56($\pm$0.05) & 0.48($\pm$0.01)\\
    WFakeGAN-GP  & 75.55($\pm$5.44) & 4.20($\pm$0.19) & 0.67($\pm$0.02) & 0.12($\pm$0.03) & 0.55($\pm$0.05) & 0.21($\pm$0.02)\\
    \midrule (Tiny-ImageNet) \\
    SFakeGAN & 80.43($\pm$5.57) & 6.97($\pm$0.40) & 0.44($\pm$0.05) & 0.13($\pm$0.02) & 0.27($\pm$0.04) & 0.20($\pm$0.02)\\
    LSFakeGAN & 78.63($\pm$5.58) & 6.86($\pm$0.53) & 0.49($\pm$0.06) & 0.10($\pm$0.02) & 0.33($\pm$0.07) & 0.21($\pm$0.04)\\
    HingeFakeGAN & 82.89($\pm$4.74) & 6.70($\pm$0.49) & 0.47($\pm$0.03) & 0.09($\pm$0.02) & 0.29($\pm$0.04) & 0.20($\pm$0.01)\\
    WFakeGAN-GP & 169.11($\pm$23.83) & 3.34($\pm$0.22) & 0.25($\pm$0.11) & 0.00($\pm$0.00) & 0.10($\pm$0.05) & 0.05($\pm$0.02)\\
    \bottomrule
  \end{tabular}}
  \end{center}
\end{table}

By comparison with Table~\ref{table:comgan}, we see that ComFakeGAN performs on par with ComGAN in all cases except for LSComFakeGAN in CIFAR10 which achieves significantly higher performance than LSComGAN. Meanwhile, in the case of ComSameGAN, we encountered training failures in most cases(see Table~\ref{table:cifar10} in Appendix). We hypothesize that the training signal available for the generator is limited in ComSameGAN since it is unable to utilize diverse information from different comparative samples.  

\subsection{Applying regularization}
We first display the results of applying equality regularization in SRGAN(standard-RGAN) and ComGAN using input concatenation architecture under $\lambda_{reg}=1.0$.

\begin{table}[H]
  \caption{ComGAN-eq results. We apply equality regularization of \eqref{eq:scomgan-eq} in SComGAN-eq and equality regularization of \eqref{eq:comgan-eq} in other algorithms. Input concatenation architecture was used except for RGAN and $\lambda_{reg}$ was set to 1.0. Values in brackets denote performance gain over ComGAN.}
  \label{table:comgan-eq}
  \begin{center}
  \resizebox{\textwidth}{!}{
  \begin{tabular}{lllllll}
    \toprule
    \textbf{Algorithms} & FID $\downarrow$ & IS $\uparrow$ & Precision $\uparrow$ & Recall $\uparrow$ & Density $\uparrow$ & Coverage $\uparrow$ \\
    \toprule (CIFAR10) \\
    SComGAN-eq & 45.71(\textcolor{green}{-0.85}) & 6.61(\textcolor{orange}{-0.21}) & 0.64(\textcolor{green}{+0.04}) & 0.42(\textcolor{green}{+0.10}) & 0.53(\textcolor{green}{+0.08}) & 0.38(+0.00)\\
    SRGAN-eq & 32.73(\textcolor{green}{-24.55}) & 7.34(\textcolor{green}{+1.02}) & 0.63(\textcolor{green}{+0.03}) & 0.54(\textcolor{green}{+0.29}) & 0.53(\textcolor{green}{+0.05}) & 0.48(\textcolor{green}{+0.18})\\
    LSComGAN-eq & 37.62(\textcolor{green}{-38.40}) & 6.87(\textcolor{green}{+1.73}) & 0.65(+0.00) & 0.46(\textcolor{green}{+0.25}) & 0.56(\textcolor{green}{+0.05}) & 0.46(\textcolor{green}{+0.20}) \\
    HingeComGAN-eq & 28.85(\textcolor{green}{-8.43}) & 7.57(\textcolor{green}{+0.12}) & 0.63(\textcolor{orange}{-0.02}) & 0.55(\textcolor{green}{+0.21}) & 0.54(\textcolor{orange}{-0.07}) & 0.53(\textcolor{green}{+0.04})\\
    WComGAN-GP-eq\tablefootnote{terminated at 32k step} & 56.71(\textcolor{green}{-22.48}) & 6.00(\textcolor{green}{+2.04}) & 0.64(\textcolor{orange}{-0.05}) & 0.36(\textcolor{green}{+0.27}) & 0.54(\textcolor{orange}{-0.05}) & 0.37(\textcolor{green}{+0.18})\\
    WComGAN-eq & 34.92(\textcolor{green}{-44.27}) & 7.14(\textcolor{green}{+3.18}) & 0.61(\textcolor{orange}{-0.08}) & 0.48(\textcolor{green}{+0.39}) & 0.50(\textcolor{orange}{-0.09}) & 0.47(\textcolor{green}{+0.28})\\
    \midrule (Tiny-ImageNet) \\
    SComGAN-eq & 77.74(\textcolor{orange}{+2.54}) & 7.05(\textcolor{orange}{-0.52}) & 0.49(\textcolor{green}{+0.06}) & 0.15(\textcolor{green}{+0.01}) & 0.33(\textcolor{green}{+0.07}) & 0.21(\textcolor{green}{+0.01})\\
    SRGAN-eq & 72.24(\textcolor{green}{-29.08}) & 6.78(\textcolor{green}{+0.40}) & 0.54(\textcolor{green}{+0.16}) & 0.18(\textcolor{green}{+0.14}) & 0.39(\textcolor{green}{+0.18}) & 0.23(\textcolor{green}{+0.09})\\
    LSComGAN-eq & 70.34(\textcolor{green}{-12.24}) & 7.47(\textcolor{green}{+0.54}) & 0.51(\textcolor{green}{+0.06}) & 0.20(\textcolor{green}{+0.09}) & 0.33(\textcolor{green}{+0.05}) & 0.24(\textcolor{green}{+0.05})\\
    HingeComGAN-eq & 78.70(\textcolor{green}{-6.74}) & 6.54(\textcolor{orange}{-0.08}) & 0.47(\textcolor{green}{+0.04}) & 0.12(\textcolor{green}{+0.04}) & 0.28(\textcolor{green}{+0.02}) & 0.20(\textcolor{green}{+0.02})\\
    WComGAN-GP-eq & 106.64(\textcolor{green}{-47.77}) & 5.04(\textcolor{green}{+1.55}) & 0.33(\textcolor{green}{+0.01}) & 0.05(\textcolor{green}{+0.04}) & 0.16(\textcolor{green}{+0.04}) & 0.11(\textcolor{green}{+0.05})\\
    WComGAN-eq & 105.58(\textcolor{green}{-48.83}) & 5.52(\textcolor{green}{+2.03}) & 0.35(\textcolor{green}{+0.03}) & 0.03(\textcolor{green}{+0.02}) & 0.18(\textcolor{green}{+0.06}) & 0.12(\textcolor{green}{+0.06})\\
    \bottomrule
  \end{tabular}}
  \end{center}
\end{table}

Table~\ref{table:comgan-eq} shows that ComGAN-eq outperforms ComGAN by a large margin in most cases demonstrating the effectiveness of equality regularization. We further show in Table~\ref{table:cifar10} and \ref{table:tinyimg} that ComGAN-eq surpasses RaGAN in all experiments except for LSRaGAN in CIFAR10 and WGAN-GP. Especially, WComGAN-eq outperforms WComGAN-GP-eq in CIFAR10, showing competitive performance with WGAN-GP, and performs on par with WComGAN-GP-eq in Tiny-ImageNet proving that equality regularization can replace gradient penalty. 

In the case of applying equality regularization in RaGAN, the results are given by Table~\ref{table:ragan-eq}.  

\begin{table}[H]
  \caption{RaGAN-eq results. We apply equality regularization of \eqref{eq:ragan-eq} under $\lambda_{reg}=1.0$ except for WGAN-GP-eq in CIFAR10 where we use $\lambda_{reg}=0.1$. Values in brackets denote performance gain over RaGAN.}
  \label{table:ragan-eq}
  \begin{center}
  \resizebox{\textwidth}{!}{
  \begin{tabular}{lllllll}
    \toprule
    \textbf{Algorithms} & FID $\downarrow$ & IS $\uparrow$ & Precision $\uparrow$ & Recall $\uparrow$ & Density $\uparrow$ & Coverage $\uparrow$ \\
    \toprule (CIFAR10) \\
    SRaGAN-eq & 32.76(\textcolor{green}{-26.23}) & 7.52(\textcolor{green}{+1.24}) & 0.62(\textcolor{green}{+0.04}) & 0.56(\textcolor{green}{+0.32}) & 0.52(\textcolor{green}{+0.10}) & 0.48(\textcolor{green}{+0.21})\\
    LSRaGAN-eq & 31.16(\textcolor{green}{-3.00}) & 7.64(\textcolor{green}{+0.09}) & 0.62(\textcolor{orange}{-0.01}) & 0.58(\textcolor{green}{+0.07}) & 0.55(\textcolor{green}{+0.01}) & 0.52(\textcolor{green}{+0.03})\\
    HingeRaGAN-eq & 25.06(\textcolor{green}{-21.25}) & 7.93(\textcolor{green}{+0.70}) & 0.66(\textcolor{green}{+0.05}) & 0.57(\textcolor{green}{+0.24}) & 0.64(\textcolor{green}{+0.13}) & 0.61(\textcolor{green}{+0.22})\\
    WGAN-GP-eq & 26.96(\textcolor{green}{-6.24}) & 7.76(\textcolor{green}{+0.69}) & 0.65(\textcolor{green}{+0.02}) & 0.60(\textcolor{green}{+0.04}) & 0.59(\textcolor{green}{+0.05}) & 0.57(\textcolor{green}{+0.08})\\
    WGAN-eq & 24.86(\textcolor{green}{-8.35}) & 7.90(\textcolor{green}{+0.83}) & 0.65(\textcolor{green}{+0.02}) & 0.58(\textcolor{green}{+0.02}) & 0.62(\textcolor{green}{+0.09}) & 0.62(\textcolor{green}{+0.13})\\
    \midrule (Tiny-ImageNet) \\ 
    SRaGAN-eq & 69.06(\textcolor{green}{-18.07}) & 7.17(\textcolor{green}{+0.48}) & 0.56(\textcolor{green}{+0.14}) & 0.20(\textcolor{green}{+0.13}) & 0.38(\textcolor{green}{+0.16}) & 0.25(\textcolor{green}{+0.08})\\
    LSRaGAN-eq & 74.29(\textcolor{orange}{+0.59}) & 6.87(\textcolor{orange}{-0.12}) & 0.52(+0.00) & 0.16(\textcolor{orange}{-0.01}) & 0.33(\textcolor{orange}{-0.03}) & 0.22(+0.00)\\
    HingeRaGAN-eq & 82.80(\textcolor{green}{-13.58}) & 6.10(\textcolor{green}{+0.05}) & 0.50(\textcolor{orange}{-0.02}) & 0.13(\textcolor{green}{+0.10}) & 0.30(\textcolor{orange}{-0.05}) & 0.20(\textcolor{green}{+0.03})\\
    WGAN-GP-eq\tablefootnote{terminated at 18k step} & 57.18(\textcolor{green}{-5.78}) & 8.23(\textcolor{green}{+0.67}) & 0.49(\textcolor{green}{+0.02}) & 0.38(\textcolor{green}{+0.09}) & 0.32(\textcolor{green}{+0.03}) & 0.29(\textcolor{green}{+0.04})\\
    WGAN-eq & 85.72(\textcolor{orange}{+22.77}) & 6.04(\textcolor{orange}{-1.52}) & 0.51(\textcolor{green}{+0.04}) & 0.11(\textcolor{orange}{-0.19}) & 0.31(\textcolor{green}{+0.02}) & 0.19(\textcolor{orange}{-0.06})\\
    \bottomrule
  \end{tabular}}
  \end{center}
\end{table}

Table~\ref{table:ragan-eq} demonstrates that applying equality regularization to RaGAN significantly improves the performance apart from LSRaGAN. This corresponds to the fact that LSRaGAN indirectly includes equality regularization in the objective. RaGAN-eq also outperforms ordinary GAN by a significant margin except for HingeGAN in Tiny-ImageNet(see Table~\ref{table:cifar10} and \ref{table:tinyimg}). Especially, the lowest FIDs in CIFAR10 and Tiny-ImageNet are respectively achieved by WGAN-eq and WGAN-GP-eq. WGAN-eq also outperforms WGAN-GP-eq in CIFAR10. We thus confirm that equality regularization renders it possible for ComGAN and RaGAN to be trained under the simplest loss of $f_1=-I$ and $f_2=I$. For the value of $\lambda_{reg}$, we observed that simply using 1.0 was sufficient in most experiments except for WGAN-GP-eq in CIFAR10 where $\lambda_{reg}=1.0$ led the discriminator to collapse to constant output due to strong regularization, which was addressed by setting $\lambda_{reg}=0.1$. Finally, in Figure~\ref{fig:cifar10-ragan}, we show the training curves of RaGAN and RaGAN-eq in CIFAR10, where we observe the stable performance of equality regularization, outperforming RaGAN throughout the whole training apart from LSRaGAN.  

\begin{figure}[H]
\centering
\setlength{\tabcolsep}{0mm}
\subfloat[SRaGAN]{
\begin{tabular}{c}
\includegraphics[width=0.24\textwidth]{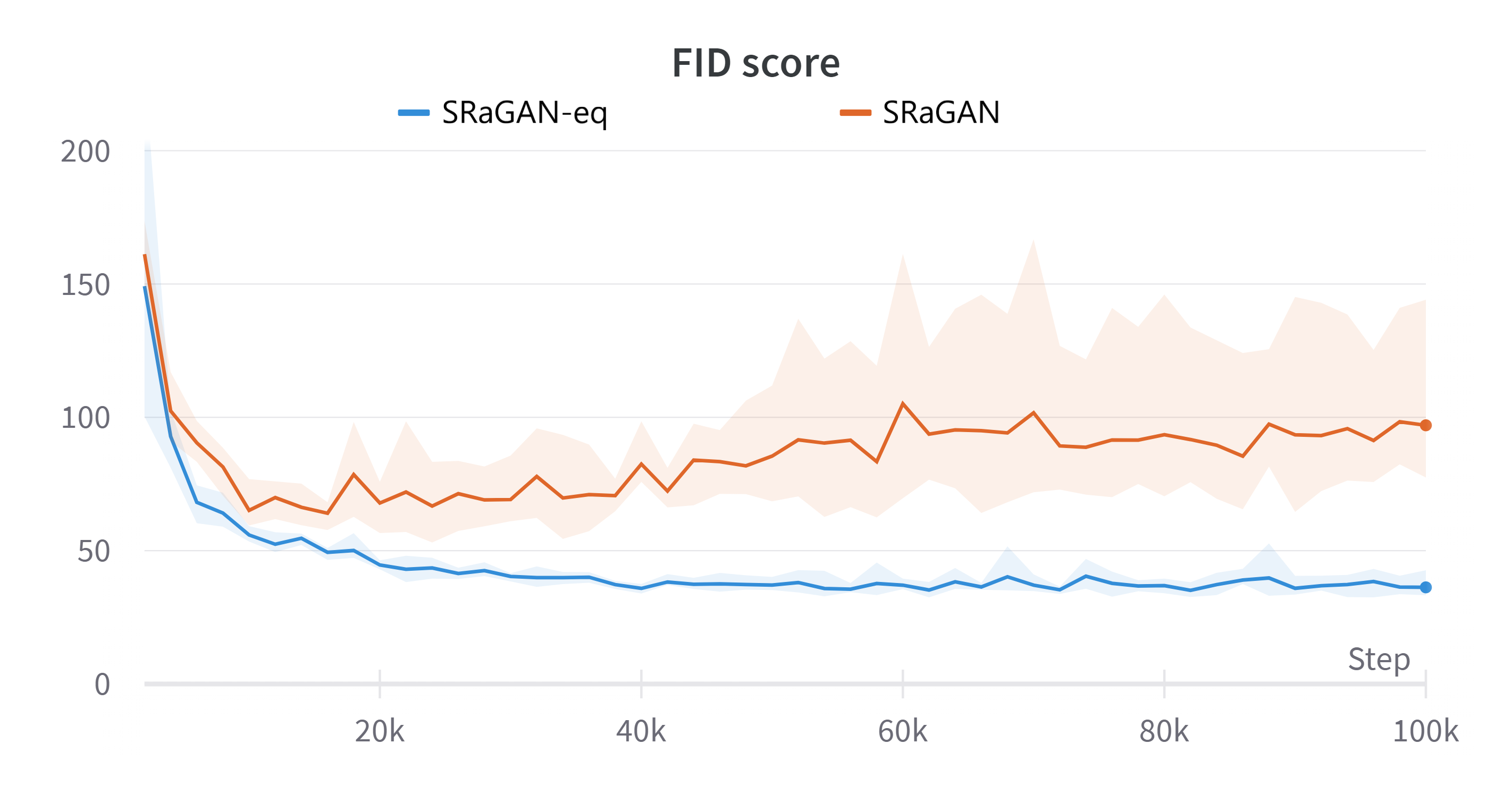} \\
\includegraphics[width=0.24\textwidth]{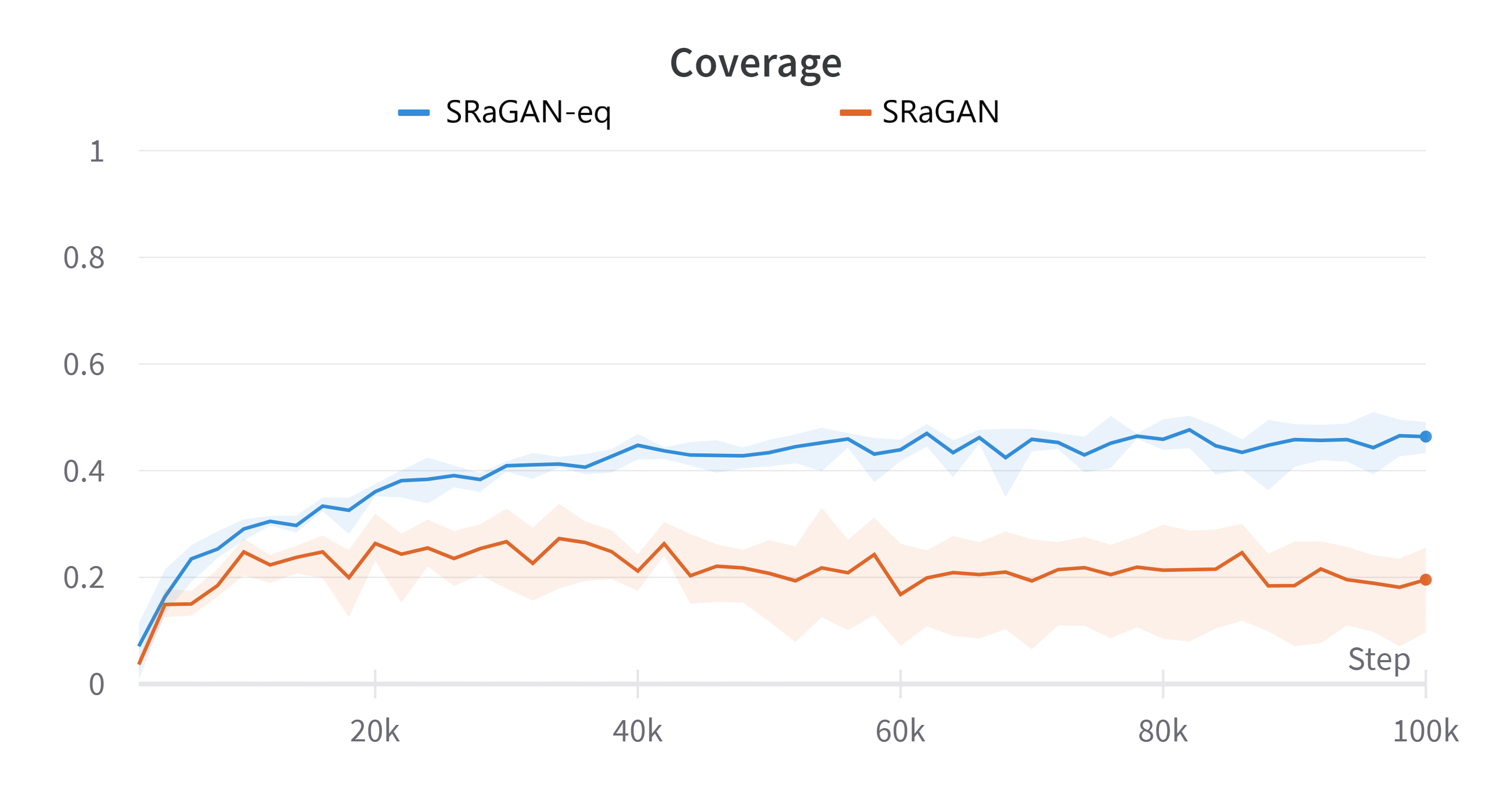}
\end{tabular}
}
\subfloat[LSRaGAN]{
\begin{tabular}{c}
\includegraphics[width=0.24\textwidth]{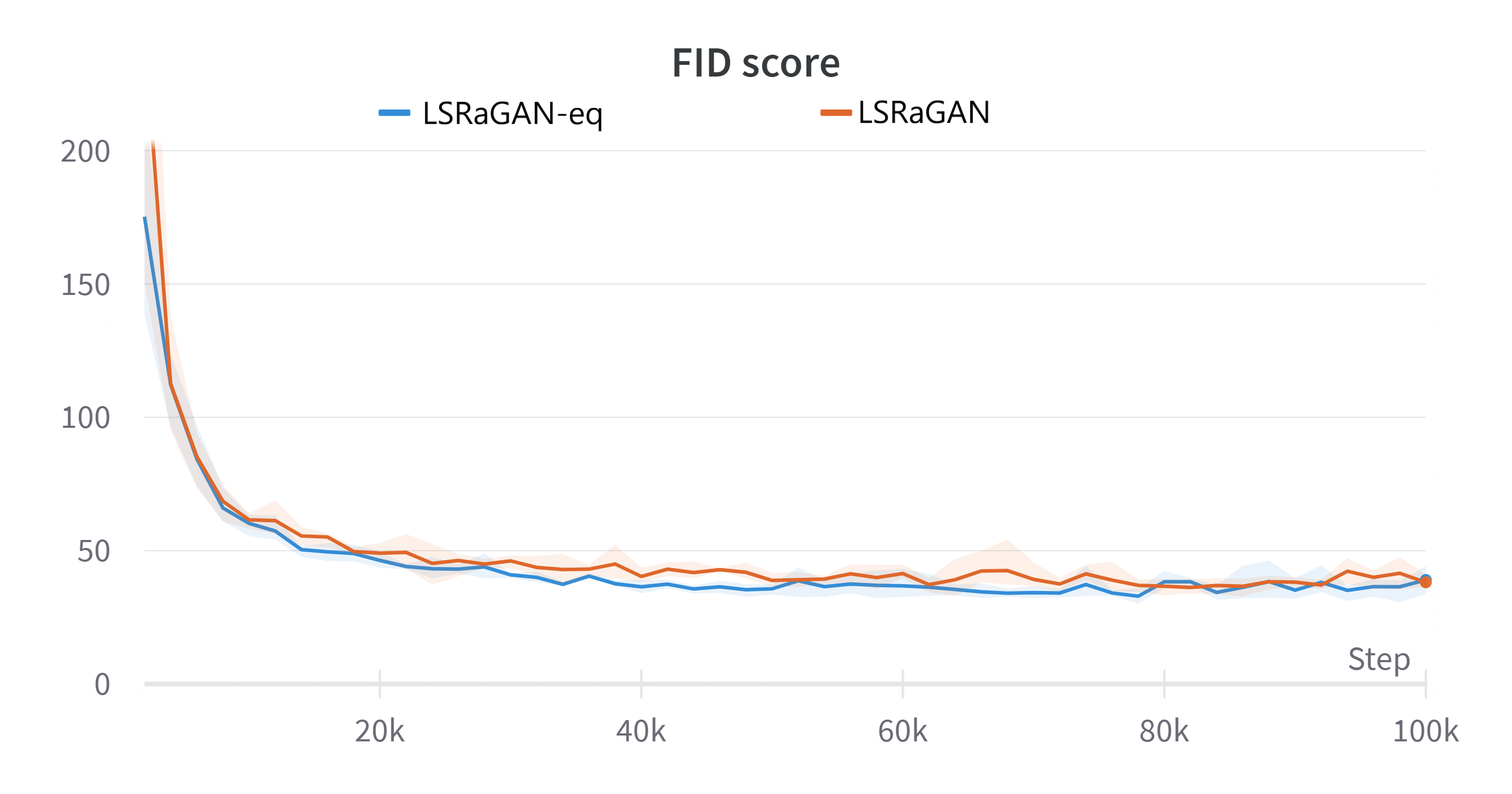} \\
\includegraphics[width=0.24\textwidth]{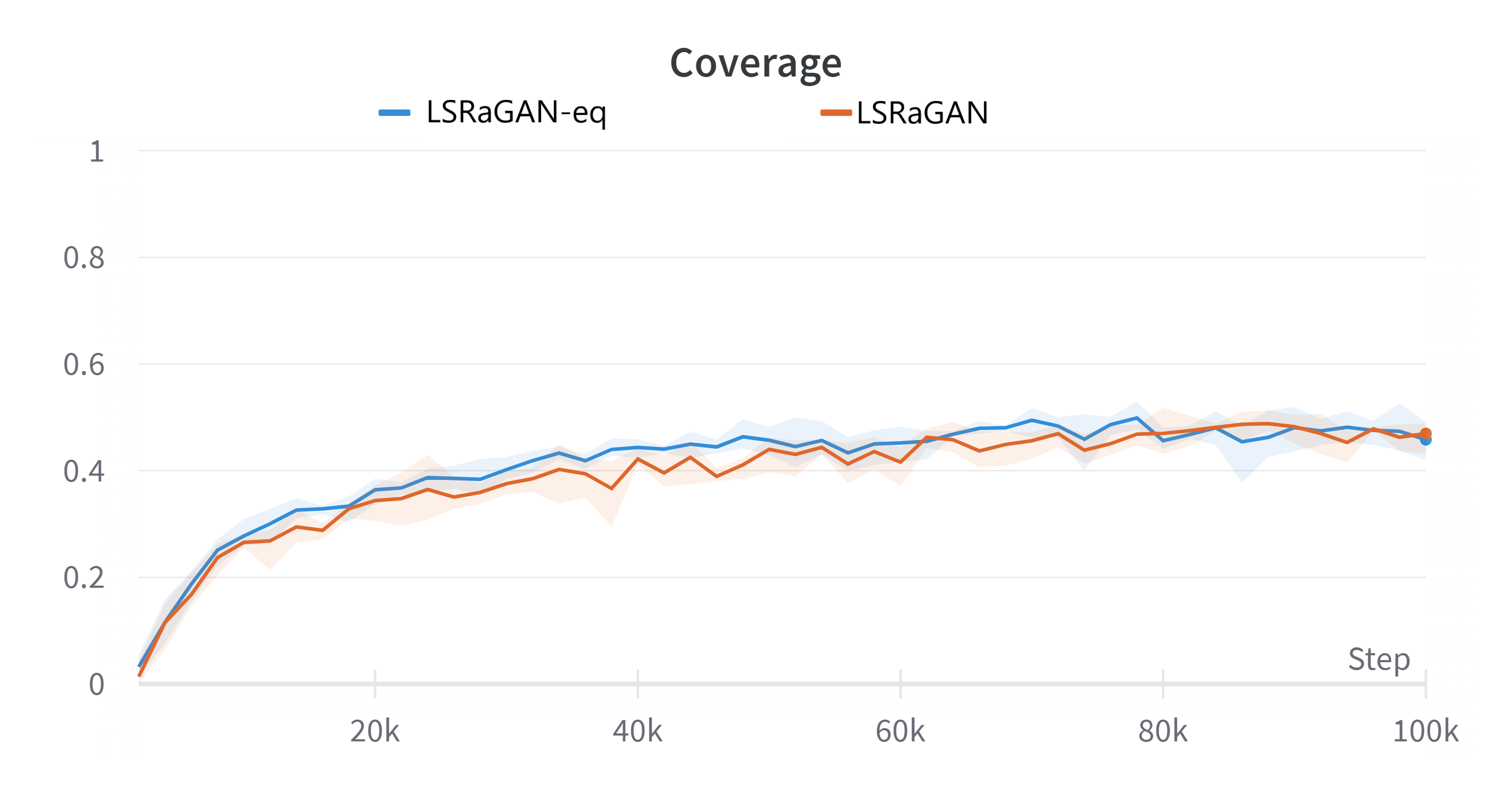}
\end{tabular}
}
\subfloat[HingeRaGAN]{
\begin{tabular}{c}
\includegraphics[width=0.24\textwidth]{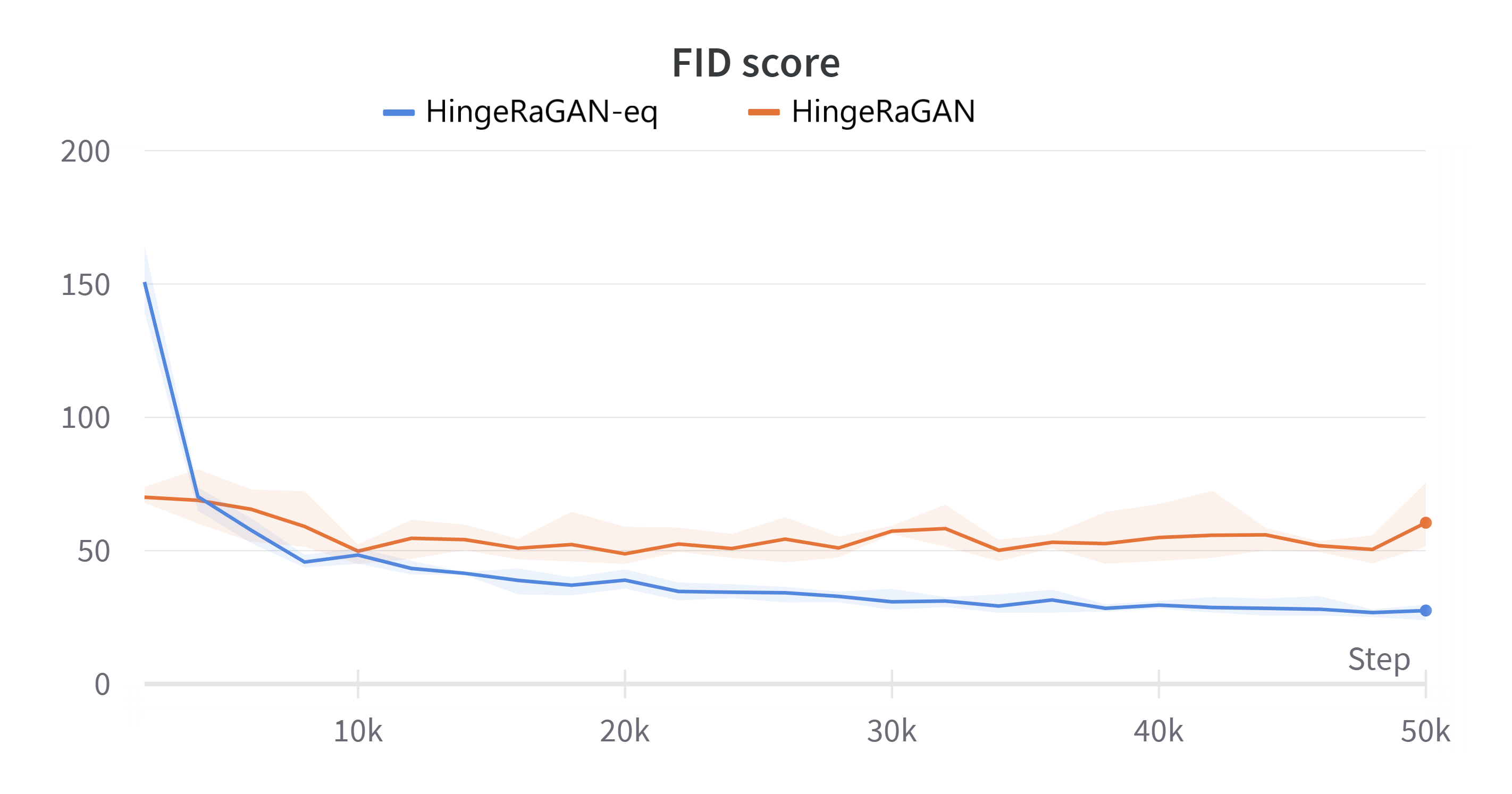} \\
\includegraphics[width=0.24\textwidth]{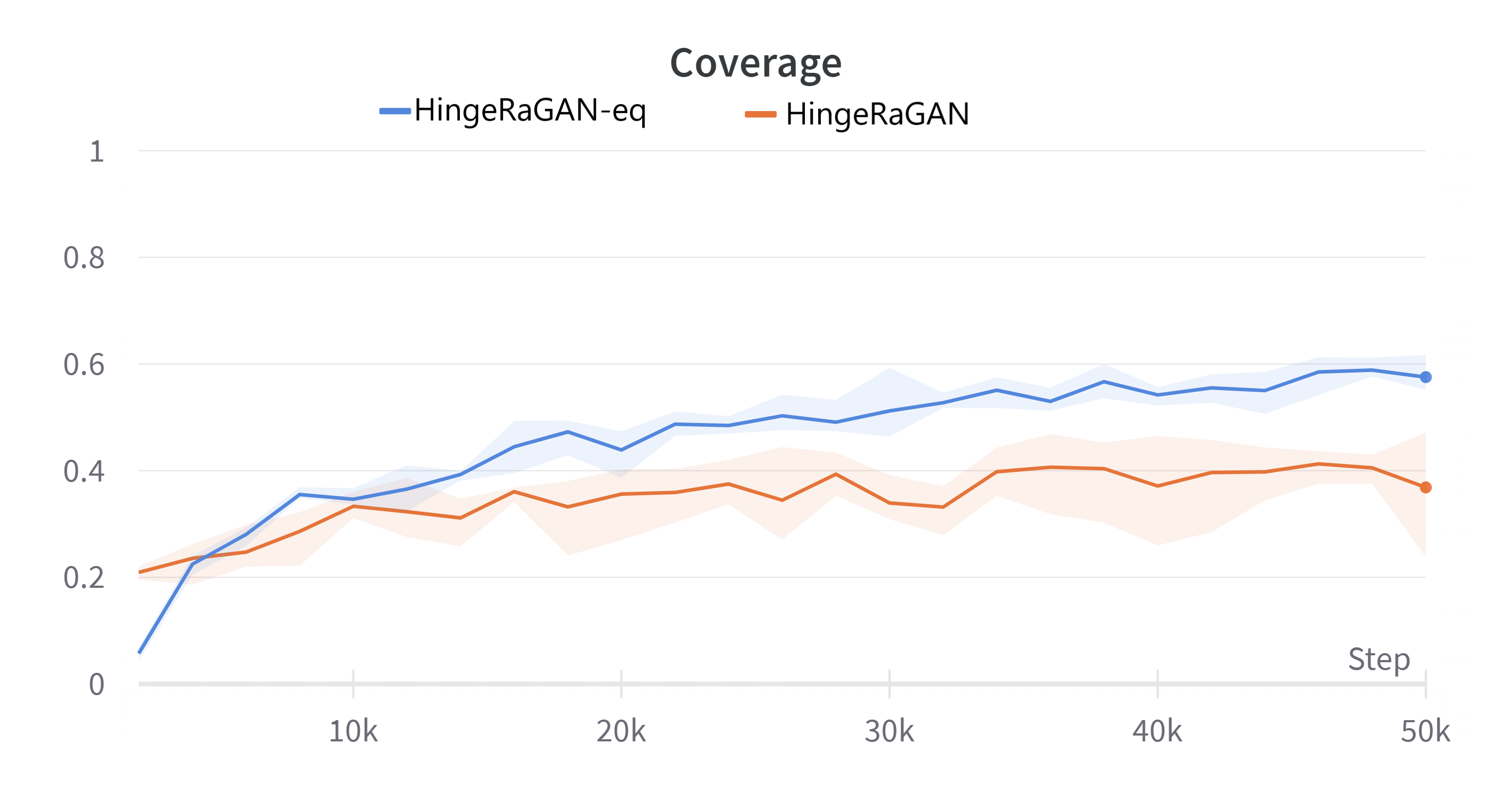}
\end{tabular}
}
\subfloat[WGAN-GP]{
\begin{tabular}{c}
\includegraphics[width=0.24\textwidth]{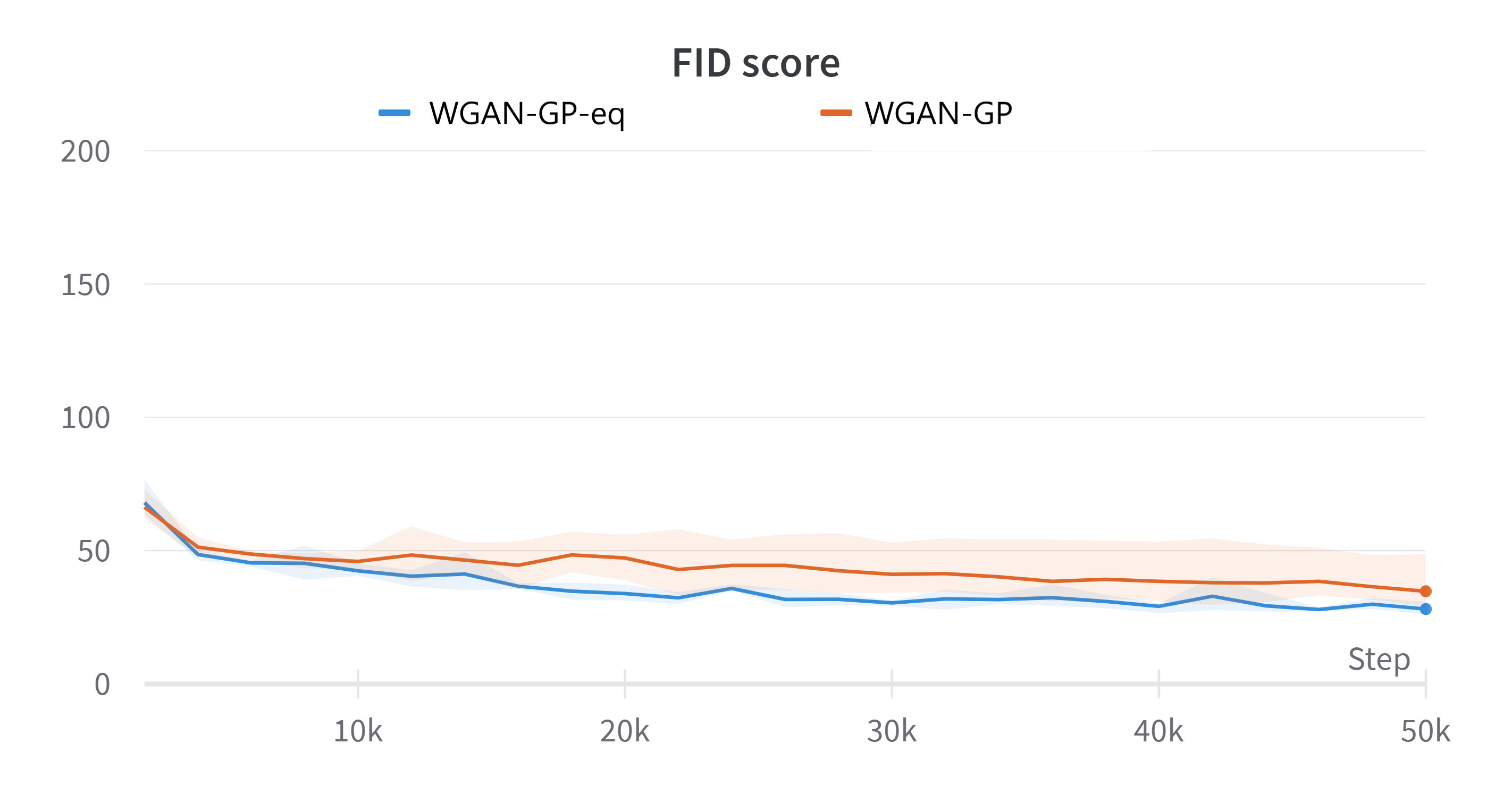} \\
\includegraphics[width=0.24\textwidth]{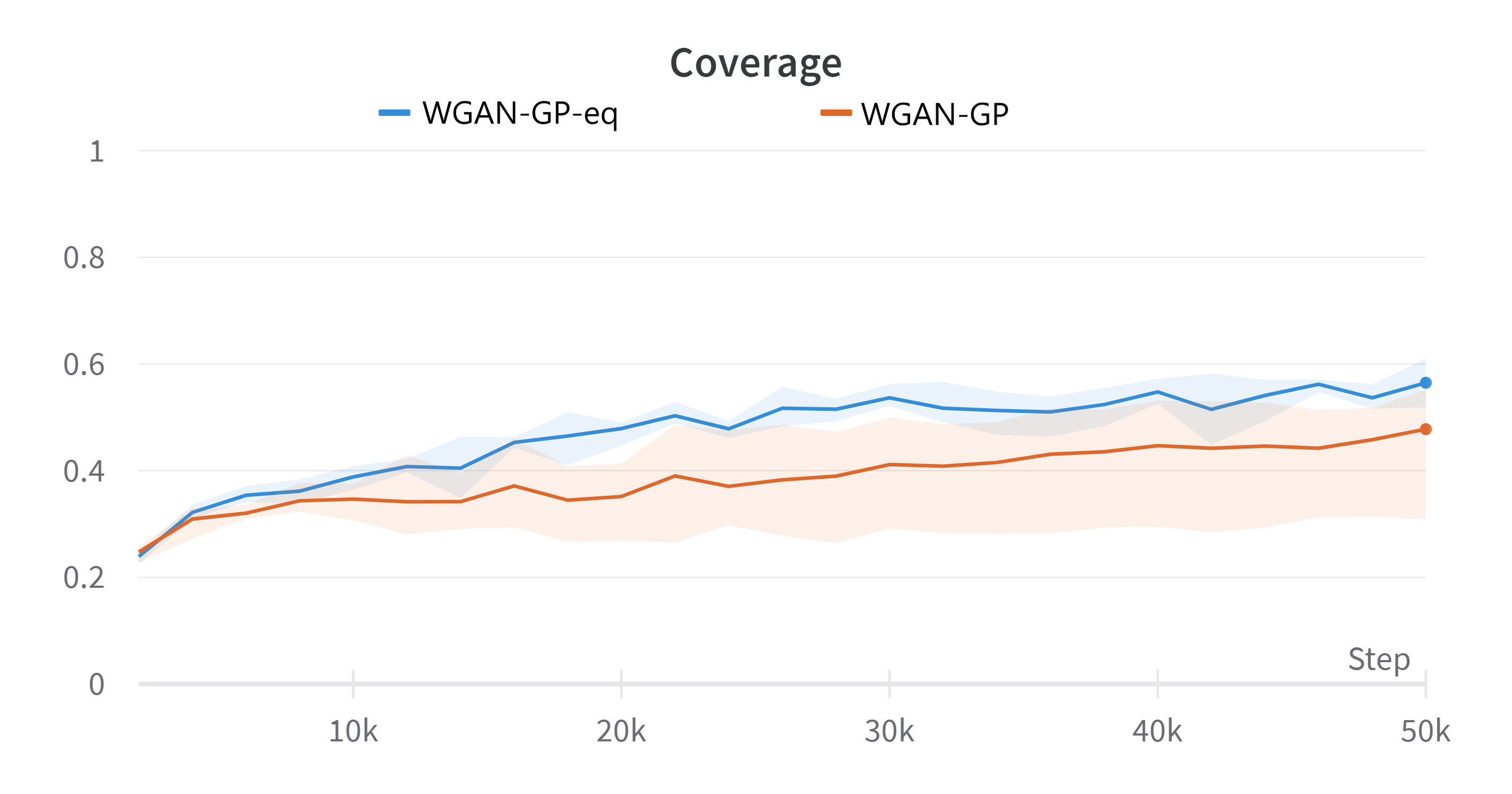}
\end{tabular}
}
\caption{FID and Coverage curves of RaGAN(orange) and RaGAN-eq(blue) in CIFAR10 dataset.}
\label{fig:cifar10-ragan}
\end{figure}

We also experimented on WGAN-rf which applies rf regularization in PacGAN under $C(x,y)=\phi(x)+\phi(y)$ and $f_1=-I$ and $f_2=I$ as shown in Table~\ref{table:wgan-rf}.

\begin{table}[H]
  \caption{WGAN-rf results. We apply equality regularization of (11) under $\lambda_{reg}=1.0$. (C) and (T) denote experiments in CIFAR10 and Tiny-ImageNet.}
  \label{table:wgan-rf}
  \begin{center}
  \resizebox{\textwidth}{!}{
  \begin{tabular}{lllllll}
    \toprule
    \textbf{Algorithms} & FID $\downarrow$ & IS $\uparrow$ & Precision $\uparrow$ & Recall $\uparrow$ & Density $\uparrow$ & Coverage $\uparrow$ \\
    \toprule
    WGAN-rf(C) & 46.46($\pm$2.25) & 6.54($\pm$0.20) & 0.61($\pm$0.02) & 0.29($\pm$0.01) & 0.48($\pm$0.03) & 0.35($\pm$0.03)\\
    WGAN-rf(T) & 106.27($\pm$4.45) & 5.65($\pm$0.20) & 0.41($\pm$0.10) & 0.03($\pm$0.01) & 0.24($\pm$0.10) & 0.13($\pm$0.02)\\
    \bottomrule
  \end{tabular}}
  \end{center}
\end{table}

From Table~\ref{table:wgan-rf}, we conclude that WGAN-rf is solely trainable without other regularization such as gradient penalty, attaining FID close to vanilla SGAN and LSGAN in CIFAR10 experiments yet the performance is still worse than WGAN-eq. We surmise that more experiments would be required to fairly compare rf-regularization and equality regularization. 

Meanwhile, since restricting the capacity of discriminators has resulted in better performance in most GANs\citep{miyato2018spectral, biggan, zhang2019consistency, lee2020regularization}, one might think that equality regularization performs well because of the same reason. To test such an argument, we apply equality regularization in ordinary GANs and present results in Table~\ref{table:gan-eq}. 

\begin{table}[H]
  \caption{GAN-eq results. We apply equality regularization of (8) under $\lambda_{reg}=1.0$. The results of CIFAR10 and Tiny-ImageNet are respectively shown in the upper rows and the lower rows. Values in brackets denote performance gain over ordinary GANs.}
  \label{table:gan-eq}
  \begin{center}
  \resizebox{\textwidth}{!}{
  \begin{tabular}{lllllll}
    \toprule
    \textbf{Algorithms} & FID $\downarrow$ & IS $\uparrow$ & Precision $\uparrow$ & Recall $\uparrow$ & Density $\uparrow$ & Coverage $\uparrow$ \\
    \toprule
    SGAN-eq & 34.74(\textcolor{green}{-15.92}) & 7.31(\textcolor{green}{+0.58}) & 0.63(\textcolor{green}{+0.03}) & 0.50(\textcolor{green}{+0.21}) & 0.53(\textcolor{green}{+0.07}) & 0.48(\textcolor{green}{+0.10})\\
    HingeGAN-eq & 54.09(\textcolor{orange}{+16.76}) & 5.73(\textcolor{orange}{-1.45}) & 0.60(\textcolor{orange}{-0.07}) & 0.34(\textcolor{green}{+0.01}) & 0.47(\textcolor{orange}{-0.17}) & 0.30(\textcolor{orange}{-0.20})\\
    \midrule
    SGAN-eq & 80.58(\textcolor{orange}{+3.64}) & 6.43(\textcolor{orange}{-0.56}) & 0.49(+0.00) & 0.14(\textcolor{green}{+0.02}) & 0.29(\textcolor{orange}{-0.04}) & 0.20(\textcolor{orange}{-0.01}) \\
    HingeGAN-eq & 95.55(\textcolor{orange}{+17.09}) & 5.95(\textcolor{orange}{-1.10}) & 0.43(\textcolor{orange}{-0.05}) & 0.05(\textcolor{orange}{-0.07}) & 0.26(\textcolor{orange}{-0.05}) & 0.15(\textcolor{orange}{-0.06}) \\
    \bottomrule
  \end{tabular}}
  \end{center}
\end{table}

Except for SGAN-eq in CIFAR10, we see that applying equality regularization in ordinary GANs actually deteriorates the performance, proving that the role of equality regularization is different from existing regularization. Such results also imply that equality regularization can promisingly be used with existing regularization(e.g. spectral normalization, r1 regularization).

\paragraph{Summary of experiments} In our experiments, RaGAN did not perform better than ordinary GANs except for LSRaGAN which indirectly possesses equality regularization. Also, applying input concatenation architecture in ComGAN surpassed RaGAN by significant margins in SGAN and HingeGAN albeit it showed worse performance in LSGAN and WGAN. For comparative samples, ComFakeGAN performed on par with ComGAN while ComSameGAN failed. With equality regularization applied, the performance of ComGAN using input-concatenation significantly improved, outperforming RaGAN in most cases. Applying equality regularization in RaGAN also improved the performances demonstrated by WGAN-eq and WGAN-GP-eq which respectively achieve the lowest FID in CIFAR10 and Tiny-ImageNet. For rf regularization, we confirmed that WGAN-rf can be trained without difficulty. Finally, applying equality regularization in ordinary GANs degraded the performances proving that the role of equality regularization is different from existing regularizations. We summarized our training results in Table~\ref{table:cifar10} and \ref{table:tinyimg} and observed RGAN-eq and RaGAN-eq variants show superb performance achieving the lowest FID among algorithms in CIFAR10 experiments and Tiny-ImageNet experiments of SGAN and WGAN. 

\section{Conclusion and Future Work}
In this work, we generalized relativistic GANs towards arbitrary architecture to enable semantic comparison. We showed training in equivalent samples is critical for performance improvements in relativistic GANs. The proposed methods, ComGAN and equality regularization outperformed relativistic GANs and ordinary GANs in various losses and evaluation metrics. Theoretically, we showed that ComGAN matches the divergences between swapped joint distributions between real and fake data, including RGAN. We also expressed optimal discriminators of proposed methods and showed connections to PacGAN. In future works, we plan to employ more advanced architecture(e.g. attentions) in ComGAN to fully exploit semantic comparisons. We also intend to analyze upon employing multiple comparative samples in ComGAN and explore rf regularization in PacGAN under various settings.

\medskip

\bibliographystyle{plainnat}
\bibliography{mybib}

\newpage 

\appendix

\section{Proofs}
\subsection{Proof of Theorem 1}
Let us express the optimal discriminator of SComGAN by $D(x,y)=\sigma(C^*(x)-C^*(y))$ where $C^*(x)=\log\frac{p_d(x)}{p_g(x)}$. Observe that 
\begin{equation}
\label{eq:kl}
    E_{z \sim{p_z}}[-C^*(G_\theta(z))]=KL(p_g \parallel p_d).
\end{equation}
Given $D^*(x,y)=\sigma(C^*(x)-C^*(y))$ where $C^*(x)=\log\frac{p_d(x)}{p_g(x)}$,
\begin{align}
\intertext{(i)}
&\nabla_\theta L^{sat}_{SComFakeGAN}=2E_{x \sim{p_g}, z \sim{p_z}}[\nabla_\theta \log\sigma(C^*(x)-C^*(G_\theta(z)))]\\
&= -2E_{x \sim{p_g}, z \sim{p_z}}[\sigma(C^*(G_\theta(z))-C^*(x))\nabla_\theta C^*(G_\theta(z)))] \\
&= -2E_{x \sim{p_g}, z \sim{p_z}}[\nabla_\theta C^*(G_\theta(z)))-\sigma(C^*(x)-C^*(G_\theta(z)))\nabla_\theta C^*(G_\theta(z)))] \\
\intertext{and}
&\nabla_\theta L^{nonsat}_{SComFakeGAN}=-2E_{x \sim{p_g}, z \sim{p_z}}[\nabla_\theta \log\sigma(C^*(G_\theta(z))-C^*(x))]\\
&=-2E_{x \sim{p_g}, z \sim{p_z}}[\sigma(C^*(x)-C^*(G_\theta(z)))\nabla_\theta C^*(G_\theta(z))].\\
\therefore \quad & \nabla_\theta L^{sat}_{SComFakeGAN} + \nabla_\theta L^{nonsat}_{SComFakeGAN}=-2E_{z\sim{p_z}}[\nabla_\theta C^*(G_\theta(z))]\\
&=2\nabla_\theta KL(p_g \parallel p_d) \; \because \; \eqref{eq:kl}  \pushQED{\qed} \qedhere \popQED
\intertext{(ii)}
&\nabla_\theta L^{sat}_{SComSameGAN}=2E_{z \sim{p_z}}[\nabla_\theta \log\sigma(C^*(\overline{G_\theta(z)})-C^*(G_\theta(z)))]\\
&=-2E_{z \sim{p_z}}[\sigma(C^*(G_\theta(z))-C^*(G_\theta(z)))\nabla_\theta C^*(G_\theta(z)))]\\
&=\nabla_\theta KL(p_g \parallel p_d) \pushQED{\qed} \qedhere \popQED \\
&\nabla_\theta L^{nonsat}_{SComSameGAN}=-2E_{z \sim{p_z}}[\nabla_\theta \log\sigma(C^*(G_\theta(z))-C^*(\overline{G_\theta(z)}))]\\
&=-2E_{z \sim{p_z}}[\sigma(C^*(G_\theta(z))-C^*(G_\theta(z)))\nabla_\theta C^*(G_\theta(z))]\\
&=\nabla_\theta KL(p_g \parallel p_d)  \pushQED{\qed} \qedhere \popQED
\intertext{where $\overline{x}$ indicates $x$ is constant not having gradients.} \notag
\end{align}
 
\subsection{Proof of Proposition 2}
Equation \eqref{eq:scomgan-eq} is equal to
\begin{equation} \begin{split}
-\iint \{p_{d,g}(x,y)+0.5p_{d,d}(x,y)+0.5p_{g,g}(x,y)\}&\log D(x,y)+\\\{p_{g,d}(x,y)+0.5p_{d,d}(x,y)+0.5p_{g,g}(x,y)\}&\log(1-D(x,y))dxdy.
\end{split} \end{equation}
By minimizing the above equation, we obtain the following optimal discriminator:
\begin{align}
D^*(x,y)&=\frac{p_{d,g}(x,y)+0.5p_{d,d}(x,y)+0.5p_{g,g}(x,y)}{p_{d,d}(x,y)+p_{d,g}(x,y)+p_{g,d}(x,y)+p_{g,g}(x,y)}\\
&=\frac{p_{d,g}(x,y)+0.5p_{d,d}(x,y)+0.5p_{g,g}(x,y)}{(p_d(x)+p_g(x))(p_d(y)+p_g(y))}.\\
\intertext{By expressing $D^*(x,y)$ with respect to $D^*(x)=\frac{p_d(x)}{p_d(x)+p_g(x)}$, it is equal to}
&=D^*(x)(1-D^*(y))+0.5D^*(x)D^*(y)+0.5(1-D^*(x))(1-D^*(y)) \\
&=0.5+0.5(D^*(x)-D^*(y)). \pushQED{\qed} \qedhere \popQED
\end{align}
Meanwhile, the objective of LSComGAN-eq in \eqref{eq:comgan-eq} is
\begin{equation} \begin{split}
\iint &p_{d,g}(x,y)(D(x,y)-1)^2+p_{g,d}(x,y)(D(x,y)+1)^2+\\&p_{d,d}(x,y)D(x,y)^2+p_{g,g}D(x,y)^2dxdy.
\end{split} \end{equation}
where $D^*(x,y)$ is given by
\begin{align} D^*(x,y)&=\frac{p_{d,g}(x,y)-p_{g,d}(x,y)}{(p_d(x)+p_g(x))(p_d(y)+p_g(y))}.\\
\intertext{By expressing $D^*(x,y)$ with respect to LSGAN optimal discriminator
$D^*(x)=\frac{p_d(x)-p_g(x)}{p_d(x)+p_g(x)}$, it is equal to}
&=0.5(D^*(x)-D^*(y)). \pushQED{\qed} \qedhere \popQED
\end{align}

\subsection{Proof of Proposition 3}
\begin{align}
\intertext{In SPacGAN-n, the optimal discriminator is given by }
&D^*(x_1 \dots x_n)=\frac{p^n_d(x_1 \dots x_n)}{p^n_d(x_1 \dots x_n)+p^n_g(x_1 \dots x_n)}\\&=\sigma \left( \log\frac{p^n_d(x_1 \dots x_n)}{p^n_g(x_1 \dots x_n)} \right)=\sigma\left(\sum_{i=1}^{n} \log\frac{p_d(x_i)}{p_g(x_i)}\right).
\intertext{Since $C^*(x)=\log\frac{p_d(x)}{p_g(x)}$, we conclude the proof. \pushQED{\qed} \qedhere \popQED}
\notag
\end{align}

\subsection{Proof of Proposition 4}
By minimizing the equation \eqref{eq:spacgan-rf} expressed by
\begin{equation} \begin{split}
-\iint \{p_{d,d}(x,y)+0.5p_{d,g}(x,y)+0.5p_{g,d}(x,y)\}&\log D(x,y)+\\\{p_{g,g}(x,y)+0.5p_{d,g}(x,y)+0.5p_{g,d}(x,y)\}&\log(1-D(x,y))dxdy,
\end{split} \end{equation}
we obtain the optimal discriminator as
\begin{align}
D^*(x,y)&=\frac{p_{d,d}(x,y)+0.5p_{d,g}(x,y)+0.5p_{g,d}(x,y)}{p_{d,d}(x,y)+p_{d,g}(x,y)+p_{g,d}(x,y)+p_{g,g}(x,y)}\\
&=\frac{p_{d,d}(x,y)+0.5p_{d,g}(x,y)+0.5p_{g,d}(x,y)}{(p_d(x)+p_g(x))(p_d(y)+p_g(y))}.\\
\intertext{By expressing $D^*(x,y)$ with respect to $D^*(x)=\frac{p_d(x)}{p_d(x)+p_g(x)}$, it is equal to}
&=D^*(x)D^*(y)+0.5D^*(x)(1-D^*(y))+0.5(1-D^*(x))D^*(y) \\
&=0.5(D^*(x)+D^*(y)). \pushQED{\qed} \qedhere \popQED
\end{align}
Meanwhile, the objective of LSPacGAN-rf in \eqref{eq:pacgan-rf} is
\begin{equation} \begin{split}
\iint &p_{d,d}(x,y)(D(x,y)-1)^2+p_{g,g}(x,y)(D(x,y)+1)^2+\\&p_{d,g}(x,y)D(x,y)^2+p_{g,d}D(x,y)^2dxdy,
\end{split} \end{equation}
where $D^*(x,y)$ is given by
\begin{align}
D^*(x,y)&=\frac{p_{d,d}(x,y)-p_{g,g}(x,y)}{(p_d(x)+p_g(x))(p_d(y)+p_g(y))}. \\
\intertext{By expressing $D^*(x,y)$ with respect to LSGAN optimal discriminator
$D^*(x)=\frac{p_d(x)-p_g(x)}{p_d(x)+p_g(x)}$, it is equal to}
&=0.5(D^*(x)+D^*(y)). \pushQED{\qed} \qedhere \popQED
\end{align}

\section{Connection to LeCam divergence}
Let us define the discriminator objective as follows:
\begin{equation}
\begin{split}
\label{eq:wlecam-eq}
L_D&=E_{x\sim{p_g}}[\phi(x)]-E_{x\sim{p_d}}[\phi(x)]+\lambda E_{x\sim{p_d}}[\|\phi(x)-E_{x\sim{p_d}}[\overline{\phi(x)}]\|^2]\\&+\lambda E_{x\sim{p_g}}[\|\phi(x)-E_{z\sim{p_z}}[\overline{\phi(G(z))}]\|^2].
\end{split}
\end{equation}
\eqref{eq:wlecam-eq} is similar to WGAN-eq except that it does not train the discriminator with respect to $E_{x\sim{p_d}}[\phi(x)]$ and $E_{z\sim{p_z}}[\phi(G(z))]$. From now on, we let $\alpha_r=E_{x\sim{p_d}}[\overline{\phi(x)}]$ and $\alpha_f=E_{z\sim{p_z}}[\overline{\phi(G(z))}]$. 
\begin{proposition}
In case that $\alpha_r=-\alpha_f$, the generator objective defined by $L_G=E_{x\sim{p_d}}[\phi^*(x)]-E_{x\sim{p_g}}[\phi^*(x)]$ is equal to $(\frac{1}{2\lambda}+\alpha_r)\Delta(p_d \parallel p_g)$ under the optimal discriminator $\phi^*(x)$ of \eqref{eq:wlecam-eq}, where $\Delta(P \parallel Q)$ is LeCam divergence\citep{le2012asymptotic} given by 
\[ \int\frac{(P(x)-Q(x))^2}{P(x)+Q(x)}dx.\]
\end{proposition}
\begin{proof}
\begin{align}
\intertext{Note that $L_D$ in \eqref{eq:wlecam-eq} can be re-expressed by}
L_D=&E_{x \sim{p_d}}[\lambda \|\phi(x)-\alpha_r\|^2-\phi(x)]+E_{x \sim{p_g}}[\lambda\|\phi(x)+\alpha_r\|^2+\phi(x)] \\ 
\begin{split}
=&E_{x \sim{p_d}}[\lambda \|\phi(x)-\alpha_r\|^2-\phi(x)+\alpha_r+\frac{1}{4\lambda}]+\\&E_{x \sim{p_g}}[\lambda\|\phi(x)+\alpha_r\|^2+\phi(x)+\alpha_r+\frac{1}{4\lambda}]+C
\end{split}\\ 
=&\lambda\int p_d(x)(\phi(x)-\alpha_r-\frac{1}{2\lambda})^2+p_g(x)(\phi(x)+\alpha_r+\frac{1}{2\lambda})^2dx+C\\
\intertext{where $C=2\alpha_r-\frac{1}{2\lambda}$. By maximizing the above equation, we obtain the optimal $\phi^*$ as}
&\phi^*(x)=\frac{(\frac{1}{2\lambda}+\alpha_r)(p_d(x)-p_g(x))}{p_d(x)+p_g(x)} \quad \because \frac{dL_D}{d\phi}=0.
\end{align}
\begin{align}
\intertext{By substituting $\phi^*$ in the generator objective,} 
&L_G=\int p_d(x) \phi^*(x)- p_g(x)\phi^*(x)dx\\
&=(\frac{1}{2\lambda}+\alpha_r)\int \frac{(p_d(x)-p_g(x))^2}{p_d(x)+p_g(x)}dx\\
&=(\frac{1}{2\lambda}+\alpha_r)\Delta(p_d \parallel p_g) \notag
\end{align}
\end{proof}

\newpage 

\section{Training losses}
\subsection{SGAN}
\begin{equation*}
\begin{alignedat}{2}
    &L^{sat}_{D}&&=E_{x\sim{p_d}}[-\log D(x)]+E_{x\sim{p_z}}[-\log(1 - D(G(z)))]\\
    &L^{sat}_{G}&&=E_{x\sim{p_d}}[\log D(x)]+E_{x\sim{p_z}}[\log(1 - D(G(z)))] \text{\qquad (Not used)}\\
    &L^{nonsat}_{G}&&=E_{x\sim{p_d}}[-\log(1-D(x))]+E_{x\sim{p_z}}[-\log D(G(z))]
\end{alignedat}
\end{equation*}
where $f_1=\log(1+e^{-x})$, $f_2=\log(1+e^x)$, and $\mathcal{A}=\sigma$.

\subsection{LSGAN}
\begin{equation*}
\begin{alignedat}{2}
    &L_{D}&&=E_{x\sim{p_d}}[(D(x)-1)^2]+E_{x\sim{p_z}}[(D(G(z))+1)^2]\\
    &_{G}&&=E_{x\sim{p_d}}[(D(x)+1)^2]+E_{x\sim{p_z}}[(D(G(z))-1)^2]
\end{alignedat}
\end{equation*}
where $f_1=(x-1)^2$, $f_2=(x+1)^2$, $g_1=f_2$, $g_2=f_1$, and $\mathcal{A}=I$.

\subsection{HingeGAN}
\begin{equation*}
\begin{alignedat}{2}
    &L_{D}&&=E_{x\sim{p_d}}[\max(0, 1-D(x))]+E_{x\sim{p_z}}[\max(0, 1+D(G(z)))]\\
    &L_{G}&&=E_{x\sim{p_d}}[D(x)]-E_{x\sim{p_z}}[D(G(z))]
\end{alignedat}
\end{equation*}
where $f_1=\max(0, 1-x)$, $f_2=\max(0, 1+x)$, $g_1=I$, $g_2=-I$, and $\mathcal{A}=I$.

\subsection{WGAN}
\begin{equation*}
\begin{alignedat}{2}
    &L_{D}&&=E_{x\sim{p_d}}[-D(x)]+E_{x\sim{p_z}}[D(G(z))]\\
    &L_{G}&&=E_{x\sim{p_d}}[D(x)]-E_{x\sim{p_z}}[D(G(z))]
\end{alignedat}
\end{equation*}
where $f_1=-I$, $f_2=I$, $g_1=I$, $g_2=-I$, and $\mathcal{A}=I$.

RGAN, RaGAN, ComGAN, and PacGAN are defined as described in previous sections.

\section{Supplementary plots}
\begin{figure}[H]
\centering
\setlength{\tabcolsep}{0mm}
\subfloat[\centering SRaGAN(C)]{
\includegraphics[width=0.3\textwidth]{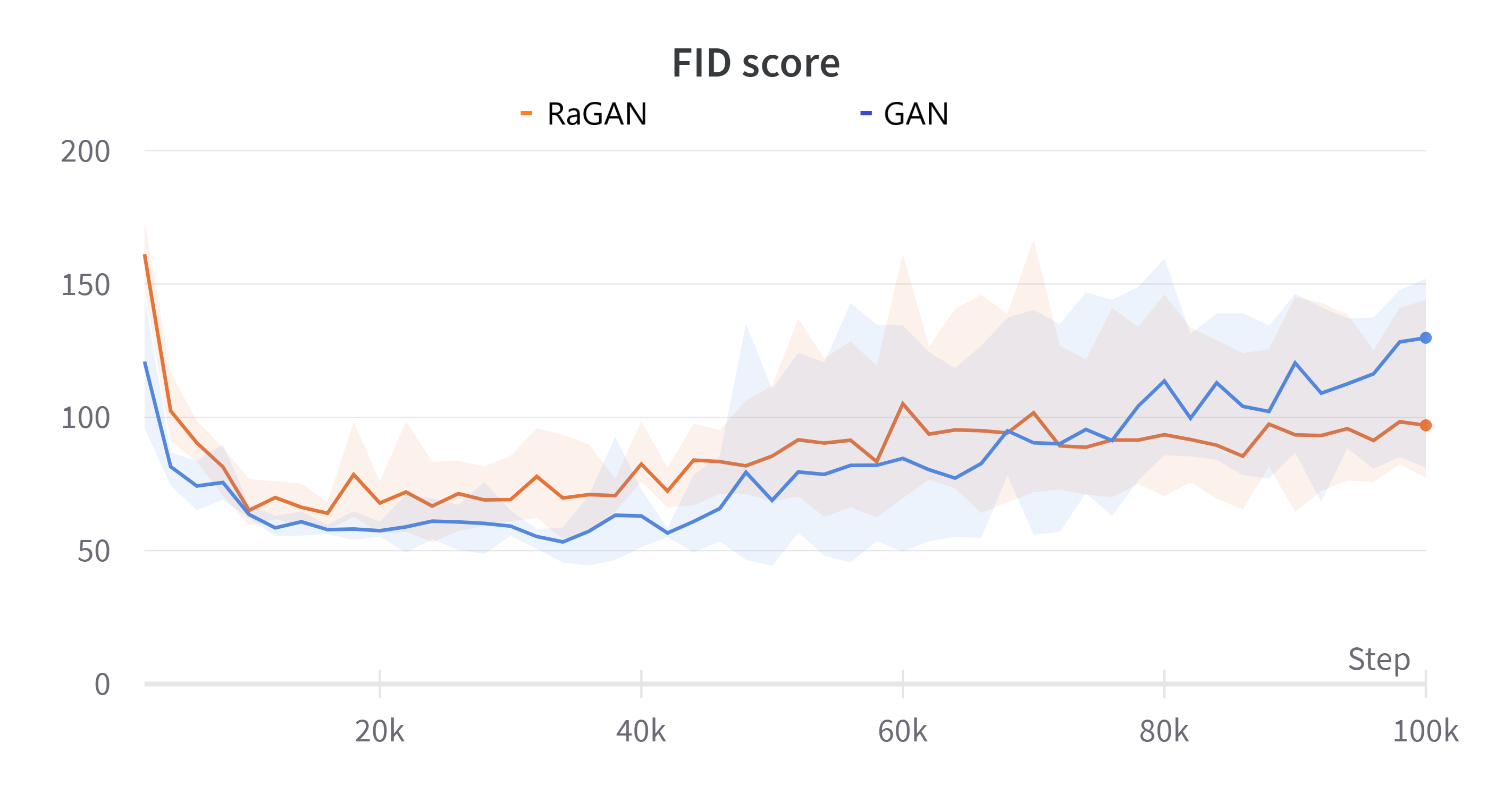} 
}
\subfloat[\centering HingeRaGAN(C)]{
\includegraphics[width=0.3\textwidth]{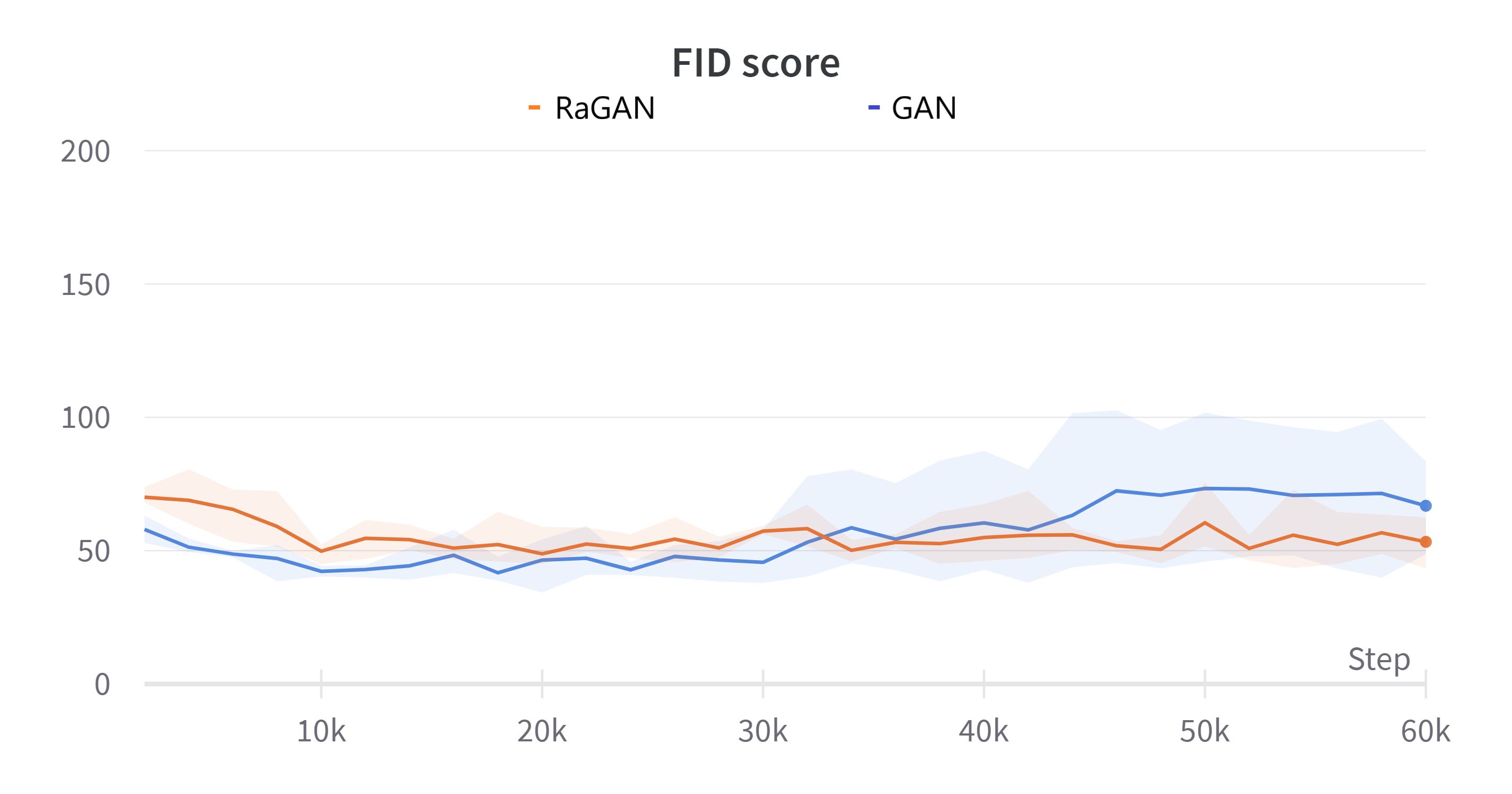} 
}
\subfloat[\centering SRaGAN(T)]{
\includegraphics[width=0.3\textwidth]{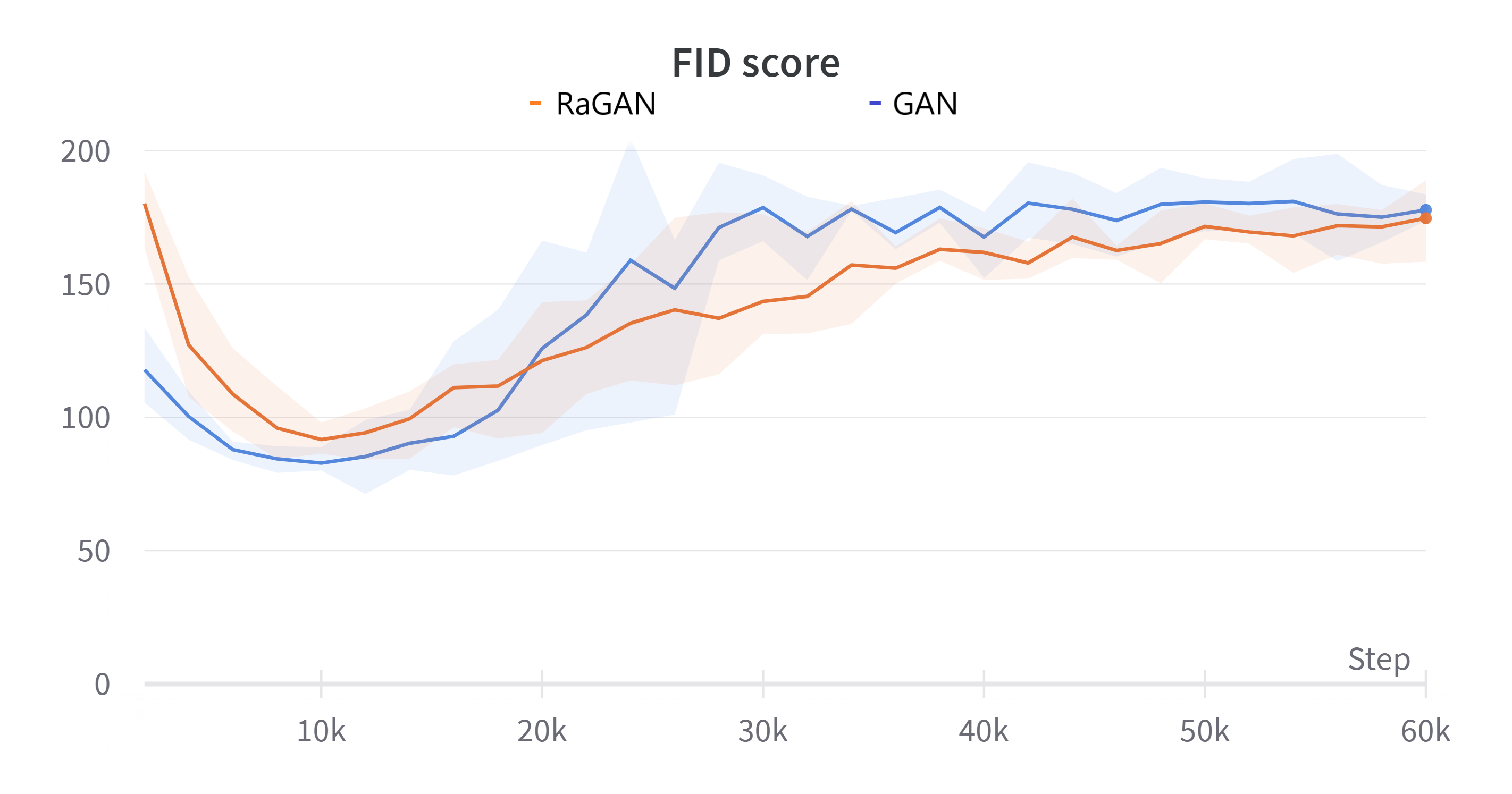} 
}
\caption{FID curves of RaGAN(orange) and ordinary GAN(blue). (C) and (T) denotes CIFAR10 and Tiny-ImageNet(We did not include experiments of HingeRaGAN in Tiny-ImageNet since it was trained for much fewer iterations(20k)).}
\label{fig:ragan-gan}
\end{figure}

\section{Performance tables}
\begin{table}[H]
  \caption{Comprehensive training results in CIFAR10 dataset\citep{cifar10}. We report the best FID and Inception Score during training averaged by 4 runs and Precision/Recall and Density/Coverage metrics at the best FID with their standard deviations. ComGAN denotes ComGAN using input concatenation architecture and FakeGAN and SameGAN denote ComFakeGAN and ComSameGAN using input concatenation architecture respectively.}
  \label{table:cifar10}
  \begin{center}
  \resizebox{\textwidth}{!}{
  \begin{tabular}{lllllll}
    \toprule
    \textbf{Algorithms} & FID $\downarrow$ & IS $\uparrow$ & Precision $\uparrow$ & Recall $\uparrow$ & Density $\uparrow$ & Coverage $\uparrow$ \\
    \toprule (training steps: 100k) \\
    SGAN & 50.65($\pm$4.65) &  6.73($\pm$0.34) & 0.60($\pm$0.03) & 0.29($\pm$0.03) & 0.46($\pm$0.04) & 0.37($\pm$0.04)\\
    SComGAN & 46.56($\pm$2.31) & 6.82($\pm$0.19) & 0.60($\pm$0.03) & 0.32($\pm$0.03) & 0.45($\pm$0.04) & 0.38($\pm$0.01)\\
    SComGAN-eq & 45.71($\pm$7.52) & 6.61($\pm$0.36) & \textbf{0.64}($\pm$0.02) & 0.42($\pm$0.07) & \textbf{0.53}($\pm$0.04) & 0.38($\pm$0.05)\\
    SFakeGAN & 45.55($\pm$4.53) & 6.57($\pm$0.33) & 0.63($\pm$0.03) & 0.33($\pm$0.03) & 0.51($\pm$0.05) & 0.39($\pm$0.05)\\
    SFakeGAN-eq & 40.18($\pm$1.80) & 6.62($\pm$0.11) & \textbf{0.64}($\pm$0.01) & 0.46($\pm$0.02) & \textbf{0.53}($\pm$0.03) & 0.42($\pm$0.02)\\
    SSameGAN & 246.58($\pm$61.19) & 2.42($\pm$0.50) & 0.47($\pm$0.21) &0.00($\pm$0.00) & 0.28($\pm$0.22) & 0.02($\pm$0.02)\\
    SRGAN & 57.28($\pm$3.95) & 6.32($\pm$0.11) & 0.59($\pm$0.01) & 0.26($\pm$0.04) & 0.48($\pm$0.03) & 0.30($\pm$0.03)\\
    SRGAN-eq & \textbf{32.73}($\pm$1.53) & 7.34($\pm$0.07) & 0.63($\pm$0.02) & 0.54($\pm$0.02) & \textbf{0.53}($\pm$0.04) & \textbf{0.48}($\pm$0.03)\\
    SRaGAN & 58.99($\pm$5.76) & 6.28($\pm$0.35) & 0.57($\pm$0.07) & 0.24($\pm$0.03) & 0.42($\pm$0.12) & 0.27($\pm$0.05)\\
    SRaGAN-eq & 32.76($\pm$0.24) & \textbf{7.52}($\pm$0.14) & 0.62($\pm$0.01) & \textbf{0.56}($\pm$0.03) & 0.52($\pm$0.04) & \textbf{0.48}($\pm$0.02)\\
    \midrule (training steps: 100k) \\
    LSGAN & 47.95($\pm$17.71) & 6.71($\pm$0.57) & 0.62($\pm$0.01) & 0.38($\pm$0.15) & 0.48($\pm$0.04) & 0.39($\pm$0.11)\\
    LSComGAN & 76.03($\pm$29.71) & 5.14($\pm$1.27) & \textbf{0.65}($\pm$0.05) & 0.22($\pm$0.19) & 0.51($\pm$0.09) & 0.25($\pm$0.13) \\
    LSComGAN-eq & 37.62($\pm$1.82) & 6.87($\pm$0.14) & \textbf{0.65}($\pm$0.01) & 0.46($\pm$0.03) & \textbf{0.56}($\pm$0.01) & 0.46($\pm$0.02) \\
    LSFakeGAN & 35.45($\pm$1.98) & 7.09($\pm$0.15) & 0.64($\pm$0.01) & 0.48($\pm$0.02) & 0.54($\pm$0.01) & 0.49($\pm$0.04) \\
    LSFakeGAN-eq & 37.15($\pm$1.80) & 6.88($\pm$0.16) & 0.64($\pm$0.01) & 0.50($\pm$0.01) & 0.55($\pm$0.02) & 0.45($\pm$0.02)\\
    LSSameGAN &  201.97($\pm$13.76) & 2.92($\pm$0.39) & 0.57($\pm$0.20) & 0.00($\pm$0.00) & 0.39($\pm$0.27) & 0.03($\pm$0.01)\\
    LSRaGAN & 34.15($\pm$1.36) & 7.55($\pm$0.15) & 0.63($\pm$0.02) & 0.51($\pm$0.01) & 0.54($\pm$0.04) & 0.50($\pm$0.02)\\
    LSRaGAN-eq & \textbf{31.16}($\pm$0.76) & \textbf{7.64}($\pm$0.13) & 0.62($\pm$0.01) & \textbf{0.58}($\pm$0.02) & 0.55($\pm$0.02) & \textbf{0.52}($\pm$0.01)\\
    \midrule (training steps: 50k) \\
    HingeGAN & 37.33($\pm$1.83) & 7.18($\pm$0.26) & \textbf{0.67}($\pm$0.02) & 0.33($\pm$0.01) & \textbf{0.64}($\pm$0.05) & 0.49($\pm$0.02)\\
    HingeComGAN & 37.28($\pm$2.29) & 7.46($\pm$0.15) & 0.65($\pm$0.03) & 0.34($\pm$0.03) & 0.62($\pm$0.07) & 0.49($\pm$0.02)\\
    HingeComGAN-eq & 28.85($\pm$1.66) & 7.57($\pm$0.12) & 0.63($\pm$0.02) & 0.55($\pm$0.01) & 0.54($\pm$0.04) & 0.53($\pm$0.03)\\
    HingeFakeGAN & 36.24($\pm$1.72) & 7.42($\pm$0.07) & 0.63($\pm$0.02) & 0.36($\pm$0.03) & 0.56($\pm$0.05) & 0.48($\pm$0.01)\\
    HingeFakeGAN-eq & 28.69($\pm$1.00) & 7.55($\pm$0.15) & 0.65($\pm$0.00) & 0.55($\pm$0.01) & 0.58($\pm$0.01) & 0.55($\pm$0.01)\\
    HingeSameGAN & 96.97($\pm$28.97) & 4.62($\pm$1.06) & 0.63($\pm$0.05) & 0.07($\pm$0.10) & 0.47($\pm$0.08) & 0.17($\pm$0.10)\\
    HingeRaGAN & 46.31($\pm$2.23) & 7.23($\pm$0.19) & 0.61($\pm$0.04) & 0.33($\pm$0.01) & 0.51($\pm$0.09) & 0.38($\pm$0.04)\\
    HingeRaGAN-eq & \textbf{25.06}($\pm$0.71) & \textbf{7.93}($\pm$0.1) & 0.66($\pm$0.02) & \textbf{0.57}($\pm$0.01) & \textbf{0.64}($\pm$0.03) & \textbf{0.61}($\pm$0.01)\\
    \midrule (training steps: 50k) \\ 
    WGAN-GP & 33.20($\pm$5.91) & 7.07($\pm$0.40) & 0.63($\pm$0.01) & 0.57($\pm$0.05) & 0.54($\pm$0.03) & 0.49($\pm$0.07)\\
    WGAN-GP-eq & 26.96($\pm$0.76) & 7.76($\pm$0.12) & 0.65($\pm$0.01) & \textbf{0.60}($\pm$0.01) & 0.59($\pm$0.01) & 0.57($\pm$0.02)\\
    \textbf{WGAN-eq} & \textbf{24.86}($\pm$0.69) & \textbf{7.90}($\pm$0.12) & 0.65($\pm$0.01) & 0.58($\pm$0.01) & \textbf{0.62}($\pm$0.02) & \textbf{0.62}($\pm$0.01)\\
    WGAN-rf & 46.46($\pm$2.25) & 6.54($\pm$0.20) & 0.61($\pm$0.02) & 0.29($\pm$0.01) & 0.48($\pm$0.03) & 0.35($\pm$0.03)\\
    WComGAN-GP & 79.19($\pm$5.49) & 3.96($\pm$0.22) & \textbf{0.69}($\pm$0.04) & 0.09($\pm$0.02) & 0.60($\pm$0.07) & 0.19($\pm$0.02)\\
    WComGAN-GP-eq\tablefootnote{terminated at 32k step} & 56.71($\pm$37.28) & 6.00($\pm$1.66) & 0.64($\pm$0.06) & 0.36($\pm$0.21) & 0.54($\pm$0.08) & 0.37($\pm$0.15)\\
    WComGAN-eq & 34.92($\pm$1.01) & 7.14($\pm$0.10) & 0.61($\pm$0.01) & 0.48($\pm$0.01) & 0.50($\pm$0.03) & 0.47($\pm$0.00)\\
    WFakeGAN-GP  & 75.55($\pm$5.44) & 4.20($\pm$0.19) & 0.67($\pm$0.02) & 0.12($\pm$0.03) & 0.55($\pm$0.05) & 0.21($\pm$0.02)\\
    WSameGAN-GP\tablefootnote{intentionally terminated at 12k step due to poor performance.} & 338.58($\pm$71.10) & 1.57($\pm$0.46) & 0.56($\pm$0.21)  & 0.00($\pm$0.00) & 0.15($\pm$0.06) & 0.00($\pm$0.01)\\
    \bottomrule
  \end{tabular}}
  \end{center}
\end{table}

\begin{table}[H]
  \caption{Comprehensive training results in Tiny-ImageNet dataset\citep{tinyimg}. We report the best FID and Inception Score during training averaged by 4 runs and Precision/Recall and Density/Coverage metrics at the best FID with their standard deviations. ComGAN denotes ComGAN using input concatenation architecture and FakeGAN and SameGAN denote ComFakeGAN and ComSameGAN using input concatenation architecture respectively.}
  \label{table:tinyimg}
  \begin{center}
  \resizebox{\textwidth}{!}{
  \begin{tabular}{lllllll}
    \toprule
    \textbf{Algorithms} & FID $\downarrow$ & IS $\uparrow$ & Precision $\uparrow$ & Recall $\uparrow$ & Density $\uparrow$ & Coverage $\uparrow$ \\
    \toprule (training steps: 60k) \\
    SGAN & 76.94($\pm$3.45) & 6.99($\pm$0.24) & 0.49($\pm$0.02) & 0.11($\pm$0.03) & 0.33($\pm$0.04) & 0.22($\pm$0.01)\\
    SComGAN & 75.20($\pm$2.07) & \textbf{7.57}($\pm$0.38) &  0.43($\pm$0.04) & 0.13($\pm$0.01) & 0.25($\pm$0.05) & 0.20($\pm$0.02)\\
    SComGAN-eq & 77.74($\pm$5.49) & 7.05($\pm$0.27) & 0.49($\pm$0.06) & 0.15	($\pm$0.04) & 0.33($\pm$0.08) &  0.21($\pm$0.04)\\
    SFakeGAN & 80.43($\pm$5.57) & 6.97($\pm$0.40) & 0.44($\pm$0.05) & 0.13($\pm$0.02) & 0.27($\pm$0.04) & 0.20($\pm$0.02)\\
    SSameGAN & 230.53($\pm$6.10) & 3.20($\pm$0.26) & 0.39($\pm$0.29) & 0.00 ($\pm$0.00) & 0.20($\pm$0.23) & 0.01($\pm$0.00)\\
    SRGAN & 101.32($\pm$2.52) & 6.38($\pm$0.23) & 0.39($\pm$0.02) & 0.03($\pm$0.01) & 0.21($\pm$0.03) & 0.14($\pm$0.01)\\
    SRGAN-eq & 72.24($\pm$0.73) & 6.78($\pm$0.11) &  0.54($\pm$0.01) & 0.18($\pm$0.01) & \textbf{0.39}($\pm$0.03) & 0.23($\pm$0.00)\\
    SRaGAN & 87.13($\pm$3.17) & 6.70($\pm$0.25) & 0.42($\pm$0.05) & 0.07($\pm$0.01) & 0.22($\pm$0.05) & 0.17($\pm$0.01)\\
    SRaGAN-eq & \textbf{69.06}($\pm$3.63) & 7.17($\pm$0.32) &  \textbf{0.56}($\pm$0.02) & \textbf{0.20}($\pm$0.04) & 0.38($\pm$0.03) & \textbf{0.25}($\pm$0.02)\\
    \midrule (training steps: 60k) \\
    LSGAN & 77.91($\pm$2.67) & 6.73($\pm$0.19) & 0.46($\pm$0.02) & 0.13($\pm$0.02) & 0.28($\pm$0.03) & 0.21($\pm$0.02)\\
    LSComGAN & 82.58($\pm$6.03) & 6.93($\pm$0.60) & 0.45($\pm$0.04) & 0.11($\pm$0.02) & 0.28($\pm$0.03) & 0.19($\pm$0.02)\\
    LSComGAN-eq & \textbf{70.34}($\pm$4.87) & \textbf{7.47}($\pm$0.34) & 0.51($\pm$0.02) & \textbf{0.20}($\pm$0.05) & 0.33($\pm$0.03) & \textbf{0.24}($\pm$0.02)\\
    LSFakeGAN & 78.63($\pm$5.58) & 6.86($\pm$0.53) & 0.49($\pm$0.06) & 0.10($\pm$0.02) & 0.33($\pm$0.07) & 0.21($\pm$0.04)\\
    LSRaGAN & 73.70($\pm$4.29) & 6.99($\pm$0.29) & \textbf{0.52}($\pm$0.03) & 0.18($\pm$0.01) & \textbf{0.36}($\pm$0.04) & 0.23($\pm$0.02)\\
    LSRaGAN-eq & 74.29($\pm$3.40) & 6.87($\pm$0.27) & \textbf{0.52}($\pm$0.03) & 0.16($\pm$0.03) & 0.33($\pm$0.03) & 0.22($\pm$0.01)\\
    \midrule (training steps: 20k) \\
    HingeGAN & \textbf{78.46}($\pm$2.02) & \textbf{7.05}($\pm$0.30) & 0.48($\pm$0.06) & 0.12($\pm$0.03) & 0.31($\pm$0.06) & \textbf{0.21}($\pm$0.01)\\
    HingeComGAN & 85.45($\pm$8.72) & 6.63($\pm$0.42) & 0.43($\pm$0.03) & 0.09($\pm$0.03) & 0.26($\pm$0.05) & 0.18($\pm$0.02)\\
    HingeComGAN-eq & 78.70($\pm$2.09) & 6.54($\pm$0.21) & 0.47($\pm$0.02) & 0.12($\pm$0.02) & 0.28($\pm$0.03) & 0.20($\pm$0.01)\\
    HingeFakeGAN & 82.89($\pm$4.74) & 6.70($\pm$0.49) & 0.47($\pm$0.03) & 0.09($\pm$0.02) & 0.29($\pm$0.04) & 0.20($\pm$0.01)\\
    HingeRaGAN & 96.38($\pm$5.87) & 6.05($\pm$0.26) & \textbf{0.52}($\pm$0.08) & 0.04($\pm$0.02) & \textbf{0.35}($\pm$0.12) & 0.17($\pm$0.03)\\
    HingeRaGAN-eq & 82.80($\pm$2.84) & 6.10($\pm$0.13) & 0.50($\pm$0.05) & \textbf{0.13}($\pm$0.01) & 0.30($\pm$0.04) & 0.20($\pm$0.01)\\
    \midrule (training steps: 20k) \\
    WGAN-GP & 62.96($\pm$1.78) & 7.56($\pm$0.21) & 0.47($\pm$0.01) & 0.29($\pm$0.02) & 0.29($\pm$0.01) & 0.25($\pm$0.01)\\
    \textbf{WGAN-GP-eq}\tablefootnote{terminated at 18k step} & \textbf{57.18}($\pm$1.93) & \textbf{8.23}($\pm$0.10) & 0.49($\pm$0.01) & \textbf{0.38}($\pm$0.02) & \textbf{0.32}($\pm$0.01) & \textbf{0.29}($\pm$0.01)\\
    WGAN-eq & 85.72($\pm$1.18) & 6.04($\pm$0.06) & \textbf{0.51}($\pm$0.01) & 0.11($\pm$0.02) & 0.31($\pm$0.01) & 0.19($\pm$0.00)\\
    WGAN-rf & 106.27($\pm$4.45) & 5.65($\pm$0.20) & 0.41($\pm$0.10) & 0.03($\pm$0.01) & 0.24($\pm$0.10) & 0.13($\pm$0.02)\\
    WComGAN-GP & 154.41($\pm$28.23) & 3.49($\pm$0.55) & 0.32($\pm$0.04) & 0.01($\pm$0.01) & 0.12($\pm$0.03) & 0.06($\pm$0.02)\\
    WComGAN-GP-eq & 106.64($\pm$6.38) & 5.04($\pm$0.35) & 0.33($\pm$0.07) & 0.05($\pm$0.02) & 0.16($\pm$0.05) & 0.11($\pm$0.02)\\
    WComGAN-eq & 105.58($\pm$4.02) & 5.52($\pm$0.31) & 0.35($\pm$0.05) & 0.03($\pm$0.02) & 0.18($\pm$0.05) & 0.12($\pm$0.01)\\
    WFakeGAN-GP & 169.11($\pm$23.83) & 3.34($\pm$0.22) & 0.25($\pm$0.11) & 0.00($\pm$0.00) & 0.10($\pm$0.05) & 0.05($\pm$0.02)\\
    \bottomrule
  \end{tabular}}
  \end{center}
\end{table}

\end{document}